\documentclass[10pt,journal,compsoc]{IEEEtran}
\ifCLASSOPTIONcompsoc
  \usepackage[nocompress]{cite}
\else
  \usepackage{cite}
\fi
\ifCLASSINFOpdf
\else
\fi
\usepackage{amsmath}
\usepackage{amssymb}
\usepackage{mathtools}
\usepackage{amsthm}

\usepackage[mathscr]{euscript}
\usepackage[T1]{fontenc}
\usepackage{algorithm}
\usepackage{algorithmic}
\usepackage{graphicx}

\usepackage{pifont}
\usepackage{subfigure}
\usepackage{booktabs} % for professional tables
\usepackage{hyperref}
\usepackage{hyperref}
\usepackage[numbers]{natbib}
\usepackage{enumitem}
\usepackage{multirow}
\usepackage{bm}
\usepackage{colortbl}
\usepackage{xcolor}
\usepackage{arydshln}
\usepackage{array}
\usepackage{caption}
\usepackage{comment}

\hypersetup{
  colorlinks   = true, %Colours links instead of ugly boxes
  urlcolor     = black, %Colour for external hyperlinks
  linkcolor    = blue, %Colour of internal links
  citecolor   =  blue %Colour of citations
}

\newcommand{\ie}{\textit{i}.\textit{e}.,}
\newcommand{\eg}{\textit{e}.\textit{g}.,}
\newcommand{\etc}{\textit{etc}.}

\definecolor{mygreen}{rgb}{0., 0.573, 0.27}
\definecolor{myred}{rgb}{0.83, 0.078, 0.353}
\definecolor{myorange}{rgb}{1.0, 0.62, 0.14}

\theoremstyle{plain}
\newtheorem{theorem}{Theorem}

\theoremstyle{definition}
\newtheorem{definition}[theorem]{Definition}
\newtheorem{assumption}[theorem]{Assumption}
\theoremstyle{remark}

\begin{document}
\title{Learning Dynamic Graph Embeddings with Neural Controlled Differential Equations}


\author{Tiexin Qin, Benjamin Walker, Terry Lyons, Hong Yan, \emph{Life Fellow, IEEE} and Haoliang Li
% \author{Michael~Shell,~\IEEEmembership{Member,~IEEE,}
%         John~Doe,~\IEEEmembership{Fellow,~OSA,}
%         and~Jane~Doe,~\IEEEmembership{Life~Fellow,~IEEE}% <-this % stops a space
\IEEEcompsocitemizethanks{\IEEEcompsocthanksitem Tiexin Qin, Hong Yan and Haoliang Li are with the Department
of Electrical and Engineering, City University of Hong Kong, Hong Kong. Email:~{tiexinqin2-c@my.cityu.edu.hk, ityan@cityu.edu.hk, haoliang.li@cityu.edu.hk}
\IEEEcompsocthanksitem Benjamin Walker and Terry Lyons are with the Mathematical Institute, University of Oxford, Oxford, United Kingdom. Email:~{benjamin.walker2@balliol.ox.ac.uk, tlyons@maths.ox.ac.uk}. \protect\hfil\break\vspace{-0.2cm}}% <-this % stops a space
\thanks{Manuscript received April 19, 2005; revised August 26, 2015. \\
(Corresponding author: Haoliang Li.)}
}


\markboth{Journal of \LaTeX\ Class Files,~Vol.~14, No.~8, August~2015}%
{Shell \MakeLowercase{\textit{et al.}}: Bare Advanced Demo of IEEEtran.cls for IEEE Computer Society Journals}



\IEEEtitleabstractindextext{%
\begin{abstract}
This paper focuses on representation learning for dynamic graphs with temporal interactions. A fundamental issue is that both the graph structure and the nodes own their own dynamics, and their blending induces intractable complexity in the temporal evolution over graphs. Drawing inspiration from the recent progress of physical dynamic models in deep neural networks, we propose~\emph{Graph Neural Controlled Differential Equations} (GN-CDEs), a continuous-time framework that jointly models node embeddings and structural dynamics by incorporating a graph enhanced neural network vector field with a time-varying graph path as the control signal.
Our framework exhibits several desirable characteristics, including the ability to express dynamics on evolving graphs without piecewise integration, the capability to calibrate trajectories with subsequent data, and robustness to missing observations. Empirical evaluation on a range of dynamic graph representation learning tasks demonstrates the effectiveness of our proposed approach in capturing the complex dynamics of dynamic graphs. Our code is available at \url{https://github.com/WonderSeven/graph-neural-cdes}.
\end{abstract}

% Note that keywords are not normally used for peer review papers.
}

\maketitle

% is not selected <OR> if conference mode is selected - because compsoc
% conference papers position the abstract like regular (non-compsoc)
\IEEEdisplaynontitleabstractindextext

% \fi
\IEEEpeerreviewmaketitle

% ======================================================================= 
\ifCLASSOPTIONcompsoc
\IEEEraisesectionheading{\section{Introduction}\label{sec:introduction}}
\else
\section{Introduction}
\label{sec:intro}
\fi
\IEEEPARstart{G}{raph} representation learning analyzes complex structured data by representing node attributes and relationships in a low-dimensional vector space. In recent years, it has attracted increasing attention owing to the prevalence of graph-structured data. The use of deep neural networks, particularly graph neural networks (GNNs), has further facilitated the ability of graph representation learning to represent nodes. For example, GNNs have been used to study social media networks~\cite{Fan2019WWW, Sankar2021WC}, protein interactions~\cite{Gainza2020Nature}, traffic flow forecasting~\cite{Lan2022ICML}, and neuroscience~\cite{Bessadok2022TPAMI}.

Many applications of graph representation learning involve temporal interactions, yet most existing methods do not consider such dynamics. As \cite{Xu2020ICLR} pointed out, ignoring the temporal evolution in dynamic graphs can result in suboptimal performance. In certain scenarios, the dynamic structure holds key insights into the system. For example, when using a pandemic model to predict the spread of infection, the evolution of social relationships due to human events (\eg~immigration, travel, education) must be taken into account~\cite{Zhong2021COVID}. In another example, malignant cells within tumors secrete proteins that influence neighboring stromal cells and create an environment conducive to their growth and metastasis~\cite{Podhajcer2008CMR}. For more examples, see~\cite{Kazemi2020JMLR}.

In this article, we focus upon the realm of dynamic graph representation learning, where explicitly modeling both the time- and node-dependent interactions is generally required to better represent temporal evolution and the dynamic nature of the data. Despite the importance, it can be rather challenging to capture both of these dynamics effectively, especially when the changes are continuous-time and nonlinear. Early work approached this by treating a dynamic graph as a sequence of snapshots \cite{Goyal2018arXiv, Kumar2019KDD, Pareja2020AAAI}, adapting existing GNNs through temporal batching. However, such methods would introduce artificial temporal boundaries. Moreover, temporal graph networks have emerged which treat the graph as an event stream and incorporate time encoding into message passing for continuous-time modeling~\citep{Rossi2020ICML, Trivedi2019ICLR, Xu2020ICLR, Cong2023ICLR}. While these methods improve dynamic modeling, they are often limited to providing node temporal representations at discrete interaction timestamps and suffer from the staleness problem~\citep{Kazemi2020JMLR}. Recently, neural differential equations have been introduced to model node and edge representations that evolve over time~\cite{Zang2020KDD, Yan2024NN, Choi2022AAAI, Liu2024TITS}. Nevertheless, these approaches still rely on a static graph structure or segments of static graphs for dynamic inference.

To address these limitations, we propose a novel and unified framework for dynamic graph representation learning that handles both the \emph{structural dynamics} and \emph{intrinsic node dynamics} simultaneously, wherein node embeddings are considered to undergo a continuous evolution over time, grounded in Neural Controlled Differential Equations (Neural CDEs)~\cite{Kidger2020CDE} built for continuous time-series modeling. Neural CDEs are a powerful concept that have the desired calibration ability with subsequent data, robustness to missing values, and memory-efficient training from  adjoint-based backpropagation. We extend this concept to dynamic graphs, which we refer to as Graph Neural Controlled Differential Equations (GN-CDEs). Specifically, we construct a continuous-time path for the evolving graph topology, which serves as the control signal driving our neural controlled differential equation. After that, we incorporate the graph structure directly into the vector field to bias the learned dynamics towards solutions that are conditioned on the current graph structure. Based on these, our model enhances its ability to capture the causal effects inherent in dynamic graph structures. Excitingly, the capability of adjusting the predicted trajectories with incoming partially observed and irregularly sampled data from Neural CDEs still holds for our model, making it promising for practical usage. It is worth noting that, GN-CDEs are a flexible framework that we can leverage to tackle node attribute prediction, dynamic node classification, and temporal link prediction tasks with minor modifications. Besides, it can easily be extended to more complex graph structures (such as directed graphs and knowledge graphs). To further demonstrate the superiority of our method, we experimentally evaluate it on node attribute prediction and link affinity prediction tasks under evolving graph structures, where our method achieves favorable results across different setups.
The contributions of this work are summarized below.

\begin{itemize}[noitemsep]

\item We introduce GN-CDEs, a generic framework for modelling dynamic graphs by encoding their structure as paths within a controlled differential equation, allowing for continuous-time representation and evolution of node embeddings.

\item We provide a theoretical analysis of our model, including guarantees on the existence and uniqueness of solutions, and establish formal equivalences and differences between GN-CDE and related Neural ODE- and CDE-based approaches.

\item To ensure scalability, we propose a simplified version of GN-CDE that retains the key multiplicative interaction between the hidden state and control path while enabling efficient computation on large graphs.

\item Experimental results verify that our proposed method achieves superior performance compared to a range of baseline methods across different graph tasks.

\end{itemize}

% ------------------------------------------------------------------
\subsection{Related Works}
\label{sec:related_work}

\noindent\textbf{Graph Embedding Learning.}~Graph representation learning, which aims to derive meaningful embeddings from the topological structure and node attributes of a graph, has been extensively studied for relational data analysis. Early works include graph factorization techniques~\citep{Ahmed2013ICWWW} and random walk-based methods~\citep{Tang2015WWW}. In line with the success of deep learning, Graph Neural Networks have emerged as powerful tools that learn node representations by aggregating information from neighboring nodes across multiple layers, achieving outstanding performance in various tasks~\citep{Welling2016ICLR, Wu2020TNNLS}. However, all these aforementioned models are designed for representation learning on static graphs, without temporal information.

In the real world, graphs are inherently dynamic rather than static. To name but a few, the interactions of users change in time on e-commerce and social platforms. However, due to the evolutionary complexity and dynamics of time-varying graphs, it is challenging to develop useful tools for dynamic graph representation learning. To address this, considerable research attention has been devoted to dynamic graph neural networks~\cite{Feng2024arXiv}. For example, TGN~\citep{Rossi2020ICML} introduces an event-based message passing framework for generating temporal embeddings, which is further enhanced by~\citep{Tjandra2024TGN2} through incorporating an identification between source and target nodes. Other representative methods, such as DyRep~\citep{Trivedi2019ICLR} and TGAT~\citep{Xu2020ICLR}, leverage self-attention mechanisms to aggregate temporal-topological neighbors. Meanwhile, approaches like CAWN~\citep{Yu2023NIPS} and GraphMixer~\cite{Cong2023ICLR} focus on learning edge embeddings for temporal link prediction tasks. Despite their successes, these event-driven models are often limited to providing node temporal representations at discrete interaction timestamps and are prone to the staleness problem~\citep{Kazemi2020JMLR}, where embeddings become outdated as the graph continues to evolve.

\noindent\textbf{Neural Differential Equations.}~Neural differential equations offer a continuous formulation for modeling temporal dynamics on hidden representations via neural network parameterized vector fields~\cite{Chen2018ODE}. One direction that bridges this framework with GNNs is continuous-depth GNNs, which utilize the integration procedure to simulate the continuous message-passing flow over graphs~\cite{Poli2021arXiv, Xhonneux2020ICML}. Beyond this, graphs have also been used to encode spatial dependencies in spatial-temporal forecasting tasks, where differential equation models capture temporal dynamics and GNNs model spatial relations~\citep{ Choi2022AAAI, Liu2025KDD}. Additionally, GNNs have been explored as learned discretizations of partial differential equation operators on unstructured meshes for simulating complex physics~\citep{Pfaff2020ICLR,Li2023NIPS,Zeng2025ICLR}. Notably, most existing works assume a fixed graph topology over time, leaving the incorporation of structural dynamics an unresolved challenge.

% ------------------------------------------------------------------
\subsection{Paper organization and notations}
The rest of this paper is structured as follows. In Section~\ref{sec:preliminary}, we briefly introduce the background knowledge. In Section~\ref{sec:method}, we describe our main model, compare it with the competitors, and provide an approximation for efficient computation. The application in several representative graph representation learning tasks is also presented. In Section~\ref{sec:experiments}, we report the empirical performance of our model in node attribute prediction task. Finally, we conclude in Section~\ref{sec:conclusion}. Complete proofs of our theoretical results follow thereafter in the appendix.

Before continuing, we introduce several notations used throughout the paper. First of all, we use lower-case letters to denote scalars, bold lower-case letters to denote vectors, and bold upper-case letters to denote matrices. For a matrix $\mathbf{X}$, we represent the $i$-th row of $\mathbf{X}$ as $\mathbf{X}^{(i)}$, and the element at the $i$-th row and $j$-th column as $\mathbf{X}^{(i,j)}$.

% ======================================================================= 
\section{Preliminary}
\label{sec:preliminary}
This section briefly reviews the basic definitions and common manners to learn graph embeddings, and then presents two typical neural differential equations.
% ------------------------------------------------------------------
\subsection{Graph Embedding Learning}
\label{sec:graph_emb_learn}
\noindent\textbf{Static graph.}~A static graph only contains a fixed topological structure. Let a static graph represented as $\mathcal{G}=\{\mathcal{V}, \mathcal{E}\}$ where $\mathcal{V}$ is the set of nodes, and $\mathcal{E} \subseteq \mathcal{V}\times \mathcal{V}$ is the set of edges. Let $v_i \in \mathcal{V}$ denote a node and $e_{ij} \in \mathcal{E}$ denote an edge between node $v_i$ and $v_j$, $i,j \in \{1,\ldots,|\mathcal{V}|\}$. Then the topology of the graph can be represented by an adjacency matrix $\mathbf{A} \in \mathbb{R}^{|\mathcal{V}| \times |\mathcal{V}|}$ where $\mathbf{A}^{(i,j)}=1$ if $e_{ij} \in \mathcal{E}$ otherwise 0. In most complex scenarios, the graph is equipped with a node attribute matrix $\mathbf{F}=\{\mathbf{F}^{(i)}\}_{i=1}^{|\mathcal{V}|}, \mathbf{F}^{(i)} \in \mathbb{R}^m$ and edge feature matrix $\mathbf{E}=\{\mathbf{E}^{(i,j)}\}_{i,j=1}^{|\mathcal{V}|}, \mathbf{E}^{(i,j)} \in \mathbb{R}^w$. Graph embedding learning for static graphs is to create an embedding $\bm{z}{(v_i)}$ for each node $v_i$ following a specified aggregation rule such that the specific local topology and node intrinsic information can be captured, formally
\begin{equation*}
    \bm{z}{(v_i)}=\sum_{j,~\mathbf{A}^{(i,j)}=1} h(\text{msg}(\mathbf{F}^{(i)}, \mathbf{F}^{(j)}, \mathbf{E}^{(i,j)}),\mathbf{F}^{(i)}),
\end{equation*}
where $\text{msg}$ and $h$ are predefined or learnable functions working for message passing and aggregation, respectively.

\noindent\textbf{Dynamic graphs.}~Dynamic graphs can be broadly categorized into \emph{continuous-time dynamic graphs} and \emph{discrete-time dynamic graphs}~\citep{Kazemi2020JMLR}.
% A discrete-time dynamic graph is a chronological sequence of static graph snapshots that are regularly sampled with a fixed time interval, while a continuous-time dynamic graph consists of graph snapshots that are irregularly sampled. 
A continuous-time dynamic graph $\mathcal{G}_t = (\mathcal{V}_t, \mathcal{E}_t)$ consists of an evolving set of nodes $\mathcal{V}_t$, edges $\mathcal{E}_t$, and potentially node and edge attributes $\mathbf{F}_t$ and $\mathbf{E}_t$. A discrete-time dynamic graph is a sequence of observations from a continuous-time dynamic graph, $\{\mathcal{G}_{t_i}\}^N_{i=0}$.
Dynamic graphs exhibit structural dynamics arising from edge addition or deletion, node addition or deletion events, and node intrinsic dynamics due to transformations in node or edge features over time. As a result, the adjacency matrix $\mathbf{A}_t$ is also time varying. In this work, we embark on representation learning for continuous-time dynamic graphs with irregularly sampled observations. We start with undirected graphs ($\mathbf{A}_t$ is symmetric) without time-varying node attributes and edge features, then we discuss the extensions to more subtle graph structures. It is noteworthy that our proposed method can also handle tasks on discrete-time dynamic graphs in a seamless manner.

% ------------------------------------------------------------------
\subsection{Neural Differential Equations}
\label{sec:neural_diff_equ}

\noindent\textbf{Neural ordinary differential equations (Neural ODEs).}~Neural ordinary differential equations~\cite{Chen2018ODE} are the continuous-depth analogue to residual neural networks. Let $f_\theta: \bm{x} \rightarrow \bm{y}$ be a function mapping with some learnable parameters $\theta$, and $\zeta_\theta$ and $\ell_\theta$ are two linear maps. Neural ODEs are defined as
\begin{equation}
\label{eq:ode}
    \bm{z}_t = \bm{z}_0 + \int_0^t f_{\theta} (\bm{z}_s) \mathrm{d}s \quad \mathrm{and} \quad \bm{z}_0 = \zeta_\theta(\bm{x}),
\end{equation}
where $\bm{y} \approx \ell_\theta(\bm{z}_T)$ can be utilized to approximate the desired output. In this formula, the solution $\bm{z}_t$ is determined by the initial condition on $\bm{z}_0$ when $\theta$ has been learned. There exists no direct way to modify the trajectory given subsequent observations, let alone tackle structural dynamics in the data generation procedure, making the plain Neural ODEs not suitable for dynamic graph setups.

\noindent\textbf{Neural controlled differential equations (Neural CDEs).}~Neural controlled differential equations ~\cite{Kidger2020CDE} are the continuous-time analogue to recurrent neural networks and provide a natural method for modeling temporal dynamics with neural networks.

Provide an irregularly sampled time series $\bm{x}=((t_0, \bm{x}_{t_0}), (t_1, \bm{x}_{t_1}), \ldots, (t_N, \bm{x}_{t_N}))$, with each $t_k \in \mathbb{R}$ the time stamp of the observation $\bm{x}_{t_k} \in \mathbb{R}^v$ and $t_0<\cdot\cdot\cdot<t_N$. Let $X:[t_0, t_N] \rightarrow \mathbb{R}^{v+1}$ be a continuous function of bounded variation with knots at $t_0, \ldots, t_N$ such that $X_{t_k}=(t_k, \bm{x}_{t_k})$. Let $f_\theta: \mathbb{R}^w \rightarrow \mathbb{R}^{w \times (v+1)}$ and $\zeta_\theta: \mathbb{R}^{v+1} \rightarrow \mathbb{R}^w$ be neural networks depending on their own learnable parameters $\theta$. Then Neural CDEs are defined as
\begin{equation}
\label{eq:cde}
    \bm{z}_t = \bm{z}_{t_0} + \int^t_{t_0} f_{\theta}(\bm{z}_s) \mathrm{d}X_s \quad \mathrm{for}~t\in(t_0,t_N],
\end{equation}
where $\bm{z}_{t_0} = \zeta_\theta(t_0, \bm{x}_{t_0})$ and $\bm{z}_t$ is the solution of the CDE. A key difference from Neural ODEs is the interpolation of observations to form a continuous path $X_s$. This allows for the natural inclusion of time-varying data into the integration process, thereby enabling the trajectory of the system to adapt according to subsequent observations.

% ======================================================================= 
\section{Methodology}
\label{sec:method}
In this section, we first present the embedding learning problem under evolving graphs. Then we introduce our proposed differential model. After that, we provide the applications to several representative graph-related tasks (\eg~node attributes prediction, dynamic node classification, temporal link prediction).

% ------------------------------------------------------------------
\subsection{Problem Setup}
Consider a dynamic graph generated following an underlying continuous procedure that we only observe a sequence of irregularly sampled graph snapshots $\mathcal{G}=\{(t_0, G_{t_0}),\ldots,(t_N, G_{t_N})\}$, with each $t_k \in \mathbb{R}$ the time stamp of the observed graph $G_{t_k}$ and $t_0<\cdot\cdot\cdot <t_N$. Among these observations, a graph snapshot $G_{t_k}=\{\mathcal{V}, \mathcal{E}\}$ is comprised of nodes $\mathcal{V}=\{v_1,\ldots,v_{|\mathcal{V}|}\}$ and edges $\mathcal{E} \subseteq \mathcal{V}\times \mathcal{V}$ (we assume all snapshots share a common node set and edge set, and omit the subscript for simplicity). Commonly, we can represent the graph topological information for graph $G_{t_k}$ via a time-specified adjacency matrix $\mathbf{A}_{t_k}\in \mathbb{R}^{|\mathcal{V}| \times |\mathcal{V}|}$ that each interaction $e_{ij} \in \mathcal{E}$ is valued in $\mathbf{A}_{t_k}$ where $i,j \in \{1,\ldots,|\mathcal{V}|\}$. Our goal is to learn a non-linear dynamical system on the dynamic graph $\mathcal{G}$ based on the observations, formally the dynamics follow the form:
\begin{equation}
\label{eq:dynamic_graph}
    \mathbf{Z}_t = \mathbf{Z}_{t_0} + \int^t_{t_0} f(\mathbf{Z}_s)\mathrm{d}\mathbf{X}_s \quad \mathrm{for}~t\in(t_0,t_N],
\end{equation}
where $\mathbf{Z}_t=\{\bm{z}_t{(v_i)}\}_{i=1}^{|\mathcal{V}|}, \mathbf{Z}_t\in\mathbb{R}^{|\mathcal{V}| \times d}$ is output node embedding matrix, $\mathbf{X}$ is an input signal path defined on $[t_0,t_N]$ which comprises the evolving topology of $\mathcal{G}$. The subscript notation here refers to function evaluation over time. Combining the model with an assigning function $\tau:\mathbb{R}^d\rightarrow\mathbb{R}^c$ on $\mathcal{G}$, such that $\mathbf{Y}_t=\tau(\mathbf{Z}_t), \mathbf{Y}_t \in \mathbb{R}^{|\mathcal{V}| \times c}$, allows applications in several machine learning tasks, such as node attribute prediction~\cite{Gao2016Nature}, dynamic node classification~\cite{Kumar2019KDD}, temporal link prediction~\cite{Cong2023ICLR}~\etc~Very recently, some works incorporate graph convolutional networks with ODEs for the continuous inference~\cite{Zang2020KDD, Yan2024NN, Choi2022AAAI}, however, they degrade to an oversimplified setup where the neighborhood for nodes remains unchanged over time, making the proposed methods impractical for usage since the structural change could yield an unignorable effect on node embeddings. To make this problem solvable and the designed methods practical, we make the following assumption:

\begin{assumption}[Continuity]
The evolving path $\mathbf{z}:[t_0,t_N]$ $\rightarrow\mathbb{R}^d$ of each node embedding is absolutely continuous.
\end{assumption}

This continuity assumption is standard for enabling differential equations~\cite{Chen2018ODE, Kidger2020CDE} and widely used by current dynamic graph approaches \cite{Zang2020KDD, Wang2020CIKM, Yan2024NN, Choi2022AAAI}.

\begin{figure}[!tbp]
\vspace{0.2cm}
    \centering
    \includegraphics[width=0.94\linewidth]{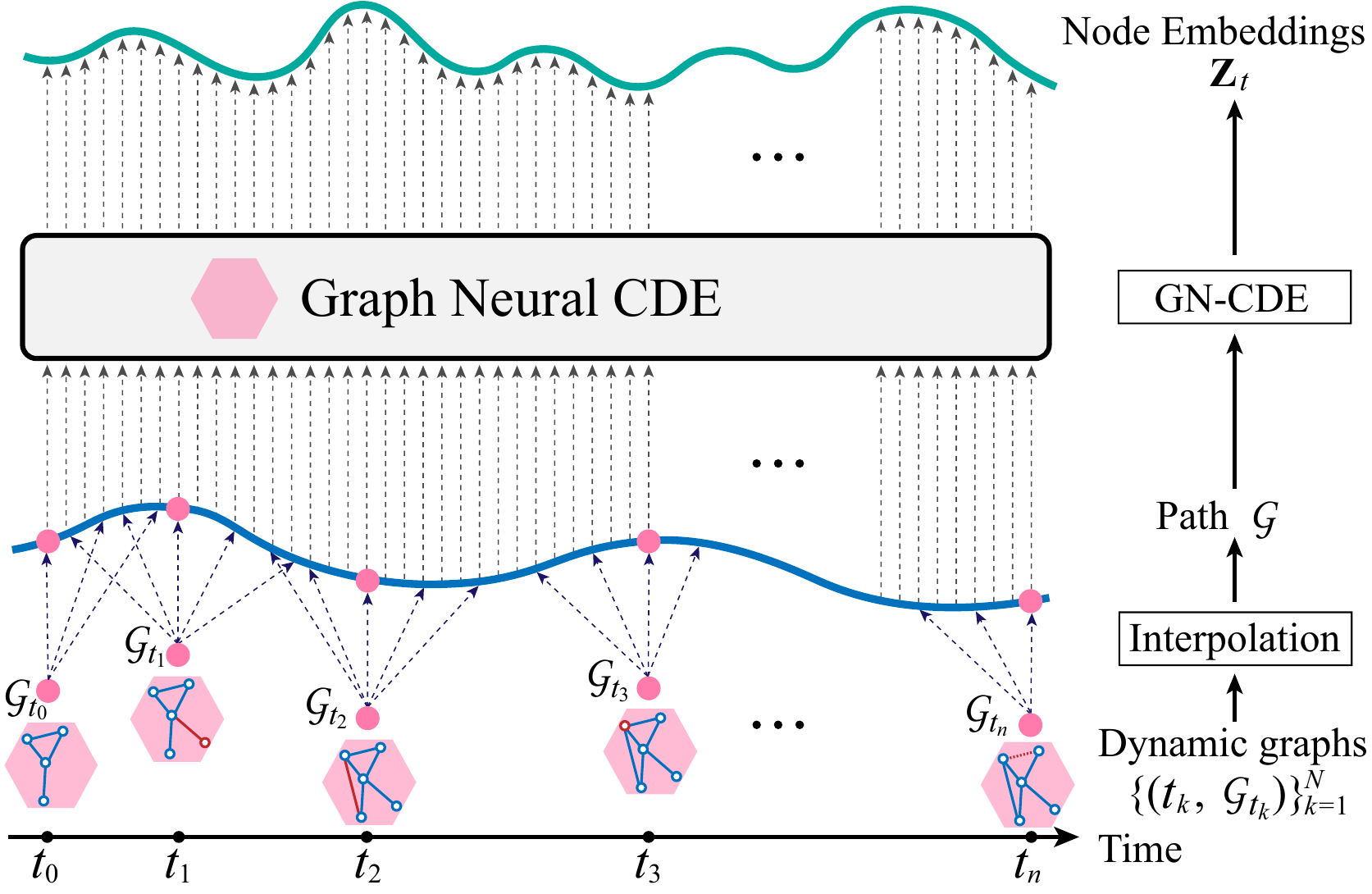}
    \vspace{-0.15cm}
    \caption{An overview of our proposed GN-CDE. In the canonical setting, dynamic graphs are characterized by an evolving adjacency matrix but can be extended to more sophisticated scenarios as discussed in Section~\ref{sec:method_expansion}.}
    \label{fig:frame}
\end{figure}

% ------------------------------------------------------------------
\subsection{GN-CDEs: Graph Neural Controlled Differential Equations}
\label{sec:gn_cde}
An illustration of the proposed GN-CDE framework can be found in Fig.~\ref{fig:frame}. Before introducing our differential equations, we need to prepare a continuously evolving path for graph structure first. Specifically, for a dynamic graph $G_{t_k}$ collected at time stamp $t_k$ endowed with adjacency matrix $\mathbf{A}_{t_k}$, we augment each interaction $e_{ij}$ in $\mathbf{A}_{t_k}$ by time stamp as $\hat{e}_{ij}=(t_k, e_{ij}) \in \mathbb{R}^2$ where $i,j \in \{1,\ldots,|\mathcal{V}|\}$  such that all these processed interactions can be represented by a time augmented adjacency matrix as $\hat{\mathbf{A}}_{t_k}=(t_k, \mathbf{A}_{t_k}) \in \mathbb{R}^{|\mathcal{V}| \times |\mathcal{V}|\times 2}$. After that, we interpolate each possible interaction among two nodes independently utilizing the discrete observations to form a continuous path, this can be represented as $\hat{\mathbf{A}}:[t_0, t_N] \rightarrow \mathbb{R}^{|\mathcal{V}| \times |\mathcal{V}|\times 2}$ such that $\hat{\mathbf{A}}_{t_k}=(t_k, \mathbf{A}_{t_k})$. In this paper, we assume $\hat{\mathbf{A}}$ to be piecewise continuously differentiable with bounded second derivative, thus many interpolation schemes can be employed~\cite{Morrill2022TMLR}.

Then, let $\zeta_{\theta}: \mathbb{R}^{|\mathcal{V}| \times |\mathcal{V}|\times 2} \rightarrow \mathbb{R}^{|\mathcal{V}| \times d}$ and $f_{\theta}: \mathbb{R}^{|\mathcal{V}| \times d} \times \mathbb{R}^{|\mathcal{V}| \times |\mathcal{V}|} \rightarrow \mathbb{R}^{|\mathcal{V}| \times d} \times \mathbb{R}^{|\mathcal{V}| \times |\mathcal{V}|\times 2}$ be two graph neural networks. We define our controlled differential equation for dynamic graphs as follows
\begin{equation}
\label{eq:gn_cde}
    \mathbf{Z}_t = \mathbf{Z}_{t_0} + \int^t_{t_0} f_{\theta}(\mathbf{Z}_s, \mathbf{A}_s)\mathrm{d}\hat{\mathbf{A}}_s \quad \mathrm{for}~t\in(t_0,t_N],
\end{equation}
where $\mathbf{Z}_{t_0}=\zeta_\theta(t_0, \mathbf{A}_{t_0})$ is treated as the initial value to avoid translational invariance. One can utilize another linear function $\ell_\theta$ to acquire the final prediction as $\tilde{\mathbf{Y}}_t=\ell_\theta(\mathbf{Z}_t), \tilde{\mathbf{Y}}_t\in\mathbb{R}^{|\mathcal{V}| \times c}$. The notation ``$f_{\theta}(\mathbf{Z}_s, \mathbf{A}_s) \mathrm{d} \hat{\mathbf{A}}_s$'' in Eq.~\eqref{eq:gn_cde} represents a tensor contraction between the final three axes of $f_{\theta}(\mathbf{Z}_s, \mathbf{A}_s)$ and all the axes of $\mathrm{d}\hat{\mathbf{A}}_s$. Our formula differs from the standard Neural CDE model~\cite{Kidger2020CDE} by explicitly capturing the causal effect of structural dynamics through the constructed graph path $\hat{\mathbf{A}}$, making it more suitable for dynamic graph scenarios, while Neural CDEs primarily emphasize sequential dependencies among observations rather than topology-driven effects. Under this modification, $\mathbf{A}_s$ in $f_{\theta}(\mathbf{Z}_s, \mathbf{A}_s)$ can bias the learned dynamics towards solutions conditioned on the current graph structure, and the derivative $\mathrm{d} \hat{\mathbf{A}}_s$ can indicate the magnitude and direction of instantaneous structural changes.

\begin{theorem}
\label{theo:existence_and_uniqueness}
    If $f_{\theta}$ is globally Lipschitz continuous and $\hat{\mathbf{A}}_t$ is a piecewise continuously differentiable interpolation, then the solution $\mathbf{Z}_t$ of \ref{eq:gn_cde} exhibits global existence and uniqueness.
\end{theorem}
\begin{proof}
    The proof is a straightforward application of the Picard–Lindelöf theorem~\citep{Coddington1955ODE}.
\end{proof}
Let $f_{\theta}$ be implemented as a Graph Neural Network of the form $f_{\theta}(\mathbf{Z}_s, \mathbf{A}_s)=\sigma(\mathbf{A}_s \mathbf{Z}_s^{(l)} \mathbf{W}^{(l)})$, where $\mathbf{W}^{(l)}$ denotes the parameters for $l$-th GCN layer and $\sigma$ is an activation function with Lipschitz constant $1$ (\eg~ReLU). Since each operation in this composition is Lipschitz continuous, it follows that $f_{\theta}$ is globally Lipschitz continuous. We elaborate on theoretical comparisons of different interpolation schemes for GN-CDE in Appendix~C. In practice, we can leverage a regularized adjacency matrix of $\mathbf{A}_s$ to stabilize the algorithm learning~\cite{Welling2016ICLR}.

\noindent\textbf{Evaluating.}~Provided $\hat{\mathbf{A}}$ as piecewise continuously differentiable, Eq.~\eqref{eq:gn_cde} can be rewritten as
\begin{equation}
\label{eq:gn_cde_deri}
    \mathbf{Z}_t = \mathbf{Z}_{t_0} + \int^t_{t_0} f_{\theta}(\mathbf{Z}_s, \mathbf{A}_s)\frac{\mathrm{d}\hat{\mathbf{A}}_s}{\mathrm{d}s}\mathrm{d}s \quad \mathrm{for}~t\in(t_0,t_N],
\end{equation}
where $\mathbf{Z}_{t_0}=\zeta_\theta(t_0, \mathbf{A}_{t_0})$. This model can be interpreted as an ordinary differential equation by taking ``$f_{\theta}(\mathbf{Z}_s, \mathbf{A}_s)\frac{\mathrm{d}\hat{\mathbf{A}}_s}{\mathrm{d}s}$'' as a whole, and one can solve it using the same techniques for Neural ODEs \citep{Kidger2020CDE}. The continuous inference procedure of GN-CDE using an ODE solver is depicted in Algorithm~\ref{alg:inference}.

\begin{algorithm}[!tbp]
\caption{Continuous inference of GN-CDE algorithm}
\label{alg:inference}
\begin{algorithmic}
   \STATE {\bfseries Input:} Sequentially observed topological structures of a dynamic graph $\{(t_0, \mathbf{A}_{t_0}),\ldots,(t_N, \mathbf{A}_{t_N})\}$, initial function $\zeta_\theta$, vector field $f_\theta$ and decoder $\ell_\theta$.
   
   \STATE {\bfseries Initializing:}
   \STATE ~~~$\triangleright$ $\hat{\mathbf{A}}$: Interpolate the time-augmented adjacency matrix;
   \STATE ~~~$\triangleright$ $\mathbf{Z}_{t_0} \leftarrow \zeta_\theta(t_0, \mathbf{A}_{t_0})$; 
   
   \STATE {\bfseries Continuously inferring:}
   \STATE ~~~$\triangleright$ $\mathbf{Z}_{t} \leftarrow \textbf{\rm{ODESolve}}(\mathbf{Z}_{t_0}, \hat{\mathbf{A}}, t_0, t_N, f_\theta) $ following Eq.~\eqref{eq:gn_cde_deri};
   \STATE ~~~$\triangleright$ $\Tilde{\mathbf{Y}}_t \leftarrow \ell_\theta(\mathbf{Z}_t)$; 
   \STATE {\bfseries Return:} $\mathbf{Z}_{t}$, $\Tilde{\mathbf{Y}}_t$
\end{algorithmic}
\end{algorithm}

% ------------------------------------------------------------------
\subsection{Properties}
\noindent\emph{Robustness to missing values.}~GN-CDEs are capable of processing partially observed data. This is because each channel of input is independently interpolated between observations, allowing $\hat{\bm{A}}_s$ to be constructed in exactly the same manner as before.

\noindent\emph{Calibration.}~The model supports calibration through the incorporation of additional observations at intermediate time points. This property is inherited from Neural CDEs~\cite{Kidger2020CDE}, which naturally allow for trajectory refinement as new data becomes available.

\noindent\emph{Decoupling of forward passes from number of observations.}: In contrast to recurrent neural networks, the number of forward passes through the vector field of a GN-CDE is not determined by the number of observations in the time series, but by the choice of differential equation solver. This makes our model robust to oversampled data.

\noindent\emph{Memory-efficient adjoint back-propagation.}~Our model can continuously incorporate incoming data without interrupting the differential equation. As a result, memory-efficient adjoint backpropagation techniques may be employed for model training.

% ------------------------------------------------------------------
\subsection{Comparison to Alternative Models}
In this section, we compare and discuss our framework with two alternatives that also combine dynamic graph structure with differential equations.

\vspace{0.3cm}
\noindent\raisebox{.5pt}{\textcircled{\raisebox{-.9pt} {1}}}~Neural CDE
\vspace{0.3cm}

An alternative of Eq.~\eqref{eq:gn_cde} is implementing the vector field without $\mathbf{A}_s$ as input, which follows the standard Neural CDE presented in Eq.~\eqref{eq:cde}. We formulize this as
\begin{equation}
\label{eq:gn_cde_abb2}
    \mathbf{Z}_t = \mathbf{Z}_{t_0} + \int^t_{t_0} f_{\theta}(\mathbf{Z}_s)\mathrm{d}\hat{\mathbf{A}}_s \quad \mathrm{for}~t\in(t_0,t_N],
\end{equation}
where $\mathbf{Z}_{t_0}=\zeta_\theta(t_0, \mathbf{A}_{t_0})$. This model emphasizes the linear dependency on $\mathrm{d}\hat{\mathbf{A}}_s$. 

\begin{theorem}
\label{theo:gn_cde_linear}
Any equation of the form $\mathbf{Z}_t = \mathbf{Z}_{t_0} + \int^t_{t_0} f_{\theta}(\mathbf{Z}_s)\mathrm{d}\hat{\mathbf{A}}_s$ can be represented exactly by a Graph Neural Controlled Differential Equation of the form $\mathbf{Z}_t = \mathbf{Z}_{t_0} + \int^t_{t_0} f_{\theta}(\mathbf{Z}_s, \mathbf{A}_s)\mathrm{d}\hat{\mathbf{A}}_s$ and vice versa.
\end{theorem}

\begin{proof}
Here we sketch the proof idea. Since the vector field $f_\theta$ in $\mathbf{Z}_t = \mathbf{Z}_{t_0} + \int^t_{t_0} f_{\theta}(\mathbf{Z}_s, \mathbf{A}_s)\mathrm{d}\hat{\mathbf{A}}_s$ takes both the $\mathbf{Z}_s$ and $\mathbf{A}_s$ as input, we can construct a stacked input $\bm\beta_s = \begin{bmatrix} \mathbf{Z}_s \\ \mathbf{A}_s\end{bmatrix}$ for alternative such that the newly built controlled differential equation on $\bm\beta_s$ is equivalent to the Neural CDE formulation presented in Eq.~\eqref{eq:gn_cde_abb2}. The full proof of Theorem~\ref{theo:gn_cde_linear} is detailed in Appendix~A.1.
\end{proof}
\vspace{0.005cm}
%Despite the equality of these two equations under some circumstances, GN-CDE with non-linear dependency (in Eq.~\ref{eq:gn_cde}) shows better experimental results compared to the linear dependency (in Eq.~\ref{eq:gn_cde_abb2}), we conjecture the reason that explicitly incorporating dynamic graph structure into the vector field can better steer the information flow among nodes for a specified time.  

Although the dynamic graph can be learned by these two forms of CDE according to Theorem~\ref{theo:gn_cde_linear}, they own different preferences during model learning stage. In the experimental part, we find that GN-CDE with non-linear dependency (as shown in Eq.~\eqref{eq:gn_cde}) performs better compared to the linear dependency variant (Eq.~\eqref{eq:gn_cde_abb2}). We conjecture that this is due to the fact that explicitly incorporating dynamic graph structure into the vector field allows for more precise control of information flow among nodes over time.

\vspace{0.3cm}
\noindent\raisebox{.5pt}{\textcircled{\raisebox{-.9pt} {1}}}~Neural ODE
\vspace{0.3cm}

Another alternative to Eq.~\eqref{eq:gn_cde} could be directly incorporating the graph structure into the vector field of Neural ODEs presented in Eq.~\eqref{eq:ode} and defining the graph neural ODE model as
\begin{equation}
\label{eq:gn_cde_abb1}
    \mathbf{Z}_t = \mathbf{Z}_{t_0} + \int^t_{t_0} f_{\theta}(\mathbf{Z}_s, \mathbf{A}_{\lfloor s \rfloor})\mathrm{d}s \quad \mathrm{for}~t\in(t_0,t_N],
\end{equation}
where $\mathbf{Z}_{t_0}=\zeta_\theta(t_0, \mathbf{A}_{t_0})$, $\mathbf{A}_{\lfloor s \rfloor}=\mathbf{A}_{t_k}$ if $t_k \leq s < t_{k+1}$.

\begin{theorem}
\label{theo:gn_ode}
Any equation of the form $\mathbf{Z}_t = \mathbf{Z}_{t_0} + \int^t_{t_0} f_{\theta}(\mathbf{Z}_s, \mathbf{A}_{\lfloor s \rfloor})\mathrm{d}s$ may be represented exactly by a Graph Neural Controlled Differential Equation of the form $\mathbf{Z}_t = \mathbf{Z}_{t_0} + \int^t_{t_0} f_{\theta}(\mathbf{Z}_s, \mathbf{A}_s)\mathrm{d}\hat{\mathbf{A}}_s$. However, the converse statement is not true.
\end{theorem}
\begin{proof}
%See Appendix~\ref{sec:app_proof_gn_ode}.
Theorem 3 allows us to identify any equation of the form $\mathbf{Z}_t = \mathbf{Z}_{t_0} + \int^t_{t_0} f_{\theta}(\mathbf{Z}_s, \mathbf{A}_s)\mathrm{d}\hat{\mathbf{A}}_s$ with an equivalent equation of the form $\mathbf{Z}_t = \mathbf{Z}_{t_0} + \int^t_{t_0} f_{\theta}(\mathbf{Z}_s)\mathrm{d}\hat{\mathbf{A}}_s$. The result then follows as a consequence of Theorem C.1 in \cite{Kidger2020CDE}.
\end{proof}

According to Theorem~\ref{theo:gn_ode}, although the Neural ODE model can also take the dynamic graph structure into the vector field computational procedure, its representation ability is inferior to our proposed GN-CDE model.

% ------------------------------------------------------------------
\subsection{Approximation of GN-CDE}
\label{sec:simplification}
Directly implementing Eq.~\eqref{eq:gn_cde_deri} by following \cite{Kidger2020CDE} will incur an undesirable computational burden due to the high dimensional tensor contraction. To ensure the scalability of our model to large neural networks, we incorporate approximations which reduce the computational complexity by leveraging a message passing mechanism. First, assuming a global time shared by all nodes, the time-augmented adjacency matrix can be simplified to $\mathrm{d}\Hat{\mathbf{A}}_s=[\mathrm{d}\mathbf{A}_s,\mathrm{d}s]$. Second, we propose to fuse $\mathrm{d}\mathbf{A}_s$ together with $\mathbf{A}_s$ into the vector field using a transformation matrix to produce a new adjacency matrix $\Tilde{\mathbf{A}}_s$, that contains both the current graph structure and the instantaneous structural change. This yields two advantages: 1) the information diffusion procedure by learnable parameters can be approximated via our adjusted adjacency matrix $\Tilde{\mathbf{A}}_s$;~2) the dimension of output for $f_\theta$ can be significantly reduced, from $\mathbb{R}^{|\mathcal{V}| \times d} \times \mathbb{R}^{|\mathcal{V}| \times |\mathcal{V}|+1}$ to $\mathbb{R}^{|\mathcal{V}| \times d}$. Formulaically, for a vector field parameterized by a $L$-layers graph neural network, the approximated equation that can be implemented much more efficiently as follows
\begin{equation}
\label{eq:gn_cde_imp}
    \mathbf{Z}_t = \mathbf{Z}_{t_0} + \int^t_{t_0} \sigma \Big( \Tilde{\mathbf{A}}_s \mathbf{Z}_s^{(L)} \mathbf{W}^{(L)}\Big)\mathrm{d}s \quad \mathrm{for}~t\in(t_0,t_N],
\end{equation}
where $\Tilde{\mathbf{A}}_s=\mathbf{W}^{(DR1)}\begin{bmatrix} \hat{\mathbf{A}}_s \\ \frac{\mathrm{d}\hat{\mathbf{A}}_s}{\mathrm{d}s} \end{bmatrix} \mathbf{W}^{(DR2)}$, and $\mathbf{W}^{(DR1)} \in \mathbb{R}^{|\mathcal{V}| \times |\mathcal{V}|}$ and $\mathbf{W}^{(DR2)} \in \mathbb{R}^{2|\mathcal{V}| \times |\mathcal{V}|}$ are transformation matrices for the fusion. Additionally, $\mathbf{Z}_s^{(L)}$ is acquired iteratively by following the rule: $\mathbf{Z}_s^{(l)} = \sigma\big(\Tilde{\mathbf{A}}_s \mathbf{Z}_s^{(l-1)} \mathbf{W}^{(l-1)}\big)$ for $l\in\{1, ..., L\}$, and $\sigma$ is ReLU activation function.

\begin{theorem}\label{theo:gn_cde_approx}
    Given a sufficient number of copies, \eqref{eq:gn_cde_imp} has the same expressivity as a Graph Neural Controlled Differential Equation model with the form $\mathbf{Z}_t = \mathbf{Z}_{t_0} + \int^t_{t_0} f_{\theta}(\mathbf{Z}_s, \mathbf{A}_s)\mathrm{d}\hat{\mathbf{A}}_s$.
\end{theorem}

\begin{proof}
    The full proof of Theorem~\ref{theo:gn_cde_approx} is deferred to Appendix~B; here we present an overview. Similarly to the proof of the universality of NCDEs \cite[Theorem B.14]{Kidger2020CDE}, the proof of Theorem \ref{theo:gn_cde_approx} proceeds by showing that you can approximate the truncated signature of the input path. Since linear maps on the signature can approximate any real-valued continuous function defined on compact sets of paths to arbitrary precision, this is sufficient to prove universality. Unlike the proof of NCDE universality, the proof of Theorem \ref{theo:gn_cde_approx} proceeds by a direct construction of the weights necessary to recreate the truncated signature, with each copy calculating one signature term.
\end{proof}

\begin{table*}[htbp]
\caption{Quantitative evaluation of prediction accuracy between GN-CDE and other baselines on node attribute prediction tasks. Here, the sum of extrapolation and interpolation results is presented for evaluating the prediction performance throughout the whole dynamic procedure. Best results are printed in boldface.}
\vspace{-0.2cm}
\label{tab:exp_attribute}
\begin{center}
\begin{normalsize}
    \begin{tabular}{p{58pt}<{\centering}p{59pt}<{\centering}p{62pt}<{\centering}p{62pt}<{\centering}p{62pt}<{\centering}p{62pt}<{\centering}p{62pt}<{\centering}}
    \toprule[1pt]
    Model & Algorithms & Grid & Random  & Power Law  & Small World  & Community \\
    \midrule[0.6pt]
    \multirow{3}*{\shortstack{Heat \\ Diffusion}} & Neural ODE     & 1.091 $\pm$ 0.344  & 0.629 $\pm$ 0.116  & 1.154 $\pm$ 0.189 & 1.093 $\pm$ 0.123  & 1.408 $\pm$ 0.091 \\
    & Neural CDE  & 0.962 $\pm$ 0.306  & 1.601 $\pm$ 0.471  & 1.642 $\pm$ 0.313 & 1.201 $\pm$ 0.179  & 1.857 $\pm$ 0.312 \\
    % & GNODE & 0.237 ± 0.322 \\ 
    & STG-NCDE       & 0.861 $\pm$ 0.541  & 1.676 $\pm$ 0.260  & 2.311 $\pm$ 0.450 & 2.491 $\pm$ 0.313  & 1.922 $\pm$ 0.216 \\
    % \cdashline{2-7}
    & \cellcolor{lightgray!20} GN-CDE         & \cellcolor{lightgray!20}\textbf{0.369} $\bm\pm$ \textbf{0.134}  & \cellcolor{lightgray!20}\textbf{0.521} $\bm\pm$ \textbf{0.202}  & \cellcolor{lightgray!20}\textbf{0.630} $\bm\pm$ \textbf{0.135} & \cellcolor{lightgray!20}\textbf{0.484} $\bm\pm$ \textbf{0.127} & \cellcolor{lightgray!20}\textbf{0.457} $\bm\pm$ \textbf{0.112} \\
    \midrule[0.6pt]
    \multirow{3}*{\shortstack{Gene \\ Regulation}} & Neural ODE     & 3.153 $\pm$ 0.562  & 3.732 $\pm$ 1.066  & 2.549 $\pm$ 0.226  & 2.252 $\pm$ 0.430  & 4.685 $\pm$ 0.759 \\
    & Neural CDE  & 2.967 $\pm$ 0.245  & 6.107 $\pm$ 3.202  & 2.764 $\pm$ 0.162  & 2.302 $\pm$ 0.591  & 5.325 $\pm$ 0.500 \\
    & STG-NCDE       & 6.554 $\pm$ 0.621  & 9.285 $\pm$ 0.808  & 3.917 $\pm$ 1.010  & 4.920 $\pm$ 1.247  & 8.278 $\pm$ 2.597 \\
    % \cdashline{2-7}
     & \cellcolor{lightgray!20} GN-CDE         & \cellcolor{lightgray!20}\textbf{1.388} $\bm\pm$ \textbf{0.262}  & \cellcolor{lightgray!20}\textbf{2.193} $\pm$ \textbf{0.550}  & \cellcolor{lightgray!20}\textbf{0.886} $\bm\pm$ \textbf{0.072}  & \cellcolor{lightgray!20}\textbf{1.331} $\bm\pm$ \textbf{0.323}  & \cellcolor{lightgray!20}\textbf{1.737} $\pm$ \textbf{0.260} \\
    \bottomrule[1pt]
    \end{tabular}
\end{normalsize}
\end{center}
\end{table*}
\vspace{-0.3cm}

In addition to the theoretical analysis of Theorem~\ref{theo:gn_cde_approx}, we also provide an empirical comparison of this simplified version in Section~\ref{sec:experiments} to evaluate its performance.

% ------------------------------------------------------------------
\subsection{Applications}
\label{sec:method_expansion}
\noindent\textbf{Node attributes prediction.}~For the node attributes prediction task, we are equipped with a node attributes matrix $\mathbf{F}_{t_k} \in \mathbb{R}^{|\mathcal{V}| \times m}$ that contains $m$-dimensional attributes of the nodes for model training. As new edges and nodes emerge at time stamp $t_k$, both $\mathbf{A}_{t_k}$ and $\mathbf{F}_{t_k}$ evolve accordingly under the effects of graph structural dynamics and intrinsic dynamics of nodes. Our goal is to predict the node attributes,~\ie~to predict $\mathbf{F}_t$ at unseen time $t$ based on previous observations. To achieve this, we need to learn informative node representations $\mathbf{Z}_t\in\mathbb{R}^{|\mathcal{V}| \times d}$ that can be used for the prediction of the nodes attributes $\mathbf{F}_t$. The objective is to minimize the following expected loss
\begin{equation}
    \min_{f_\theta, \ell_\theta}~\mathbb{E}_{t} [\mathrm{Loss}(\mathbf{F}_t, \tilde{\mathbf{Y}}_t)] \quad \mathrm{for}~t\in(t_0,T],
\end{equation}
where $\tilde{\mathbf{Y}}_t=\ell_\theta(\mathbf{Z}_t)$ is the prediction based on $\mathbf{Z}_t$ which is inferred by Eq.~\eqref{eq:gn_cde}. We can use the (mean) absolute error or squared error to measure the mismatch between $\mathbf{F}_t$ and $\tilde{\mathbf{Y}}_t$. Here, $T$ can be either $T \leq t_N$ which corresponds to interpolation prediction or $T > t_N$ which corresponds to extrapolation prediction.

\noindent\textbf{Dynamic node classification.}~The task of node classification is leveraging the collected information at time stamps $\{t_0, \ldots, t_N\}$ to predict the label of nodes $\tilde{\mathbf{Y}}_t$ at time stamp $t, t>t_N$. The objective function for minimization is
\begin{equation}
    \min_{f_\theta, \ell_\theta}~\mathbb{E}_{t} [\mathrm{Loss}(\mathbf{Y}_t, \tilde{\mathbf{Y}}_t)] \quad \mathrm{for}~t\in(t_N,T],
\end{equation}
where $\tilde{\mathbf{Y}}_t=\ell_\theta(\mathbf{Z}_t)$, $\ell_\theta$ is a MLP followed by a softmax activation function to obtain the class probability, and we can employ the cross-entropy loss for the measurement of prediction error. Commonly, this works on attributed graphs, which means the node attribute matrix $\mathbf{F}_{t_k}$ and edge feature matrix $\mathbf{E}$ are provided in advance. Thus, the controlled differential equation can be written as
\begin{equation}
\label{eq:gn_cde_node}
    \mathbf{Z}_t = \mathbf{Z}_{t_0} + \int^t_{t_0} f_{\theta}\Big(\mathbf{Z}_s, \mathbf{A}_s, \mathbf{F}_s, \mathbf{E}_s\Big) \mathrm{d}\hat{\mathbf{A}}_s \quad \mathrm{for}~t\in(t_0,t_N],
\end{equation}
where $\mathbf{Z}_{t_0}=\zeta_\theta(t_0, \mathbf{A}_{t_0}, \mathbf{F}_{t_0}, \mathbf{E}_{t_0})$. In this equation, we need to interpolate the time augmented adjacency matrix as before and the node attributes $\mathbf{F}_{t_k}$ and edge attributes $\mathbf{E}_{t_k}$ to conduct the integral. Moreover, when $\mathbf{F}_{t_k}$ or $\mathbf{E}_{t_k}$ is not given, we can disable these terms in Eq.~\eqref{eq:gn_cde_node} accordingly.

\noindent\textbf{Temporal link prediction.}~This task is to predict the existence of an edge $e_{ij}$ at unseen future time $t, t>t_N$. The node attribute matrix $\mathbf{F}_{t_k}$ and edge feature matrix $\mathbf{E}$ are also provided, thus we can follow the same setup designed for dynamic node classification tasks. To acquire the link probability among two nodes, we apply a MLP over the concatenation of the corresponding nodes' embeddings.

% ==================================================================
\section{Experiments}
\label{sec:experiments}
In this section, we conduct a comprehensive set of experiments on node attribute prediction tasks to validate the effectiveness of our proposed GN-CDE model. More details of dataset construction and experimental setup can be found in the supplementary material.
% ------------------------------------------------------------------
\subsection{Dynamic Node Property Prediction}
\noindent\textbf{Experimental Setup.}~
We consider two representative dynamic models: heat diffusion dynamics and gene regulatory dynamics. The underlying networks own 400 nodes and are initialized as Grid network, Random network, Power-law network, Small world network, and Community network, respectively. Next, some edges are randomly dropped or added occasionally to simulate dynamic environments. We irregularly sample $120$ snapshots from the continuous-time dynamics to form the entire observations. The standard data splits presented in \cite{Zang2020KDD} are utilized, with $80$ snapshots used for training, $20$ snapshots for testing the interpolation prediction task, and $20$ snapshots for testing the extrapolation prediction task.

\begin{figure*}[htbp]
	\centering
	\begin{minipage}{\linewidth}
		\centering
		\includegraphics[width=0.98\linewidth]{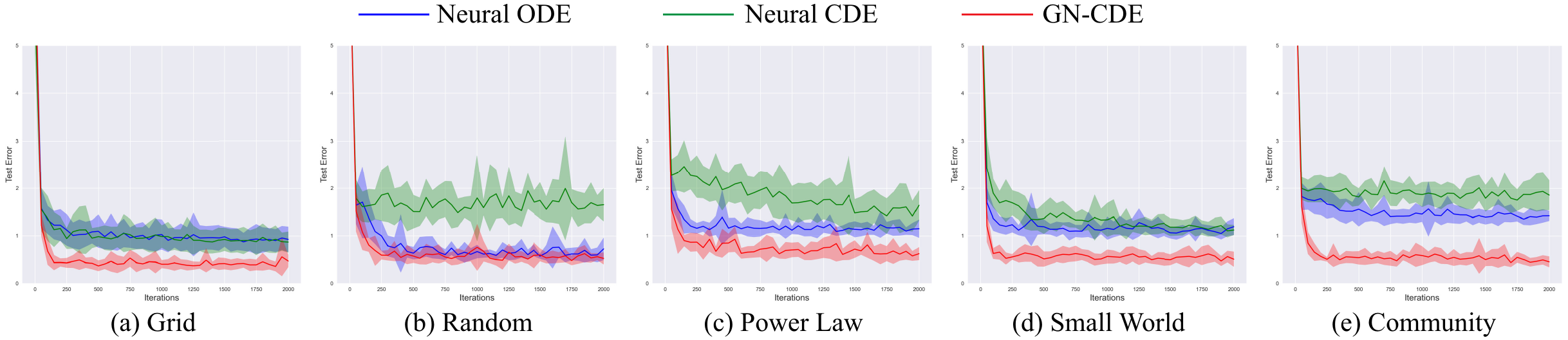}
        \vspace{-0.2cm}
		\caption{\textbf{Heat Diffusion:}~The test errors of Neural ODE, Neural CDE and our GN-CDE models with respect to the optimization iteration count under five different graph structures: (a) grid, (b) Random, (c) power law, (d) small world and (e) community.}
		\label{fig:train_curve_heat}
        \vspace{0.3cm}
	\end{minipage}
    \\
    \begin{minipage}{\linewidth}
		\centering
		\includegraphics[width=0.98\linewidth]{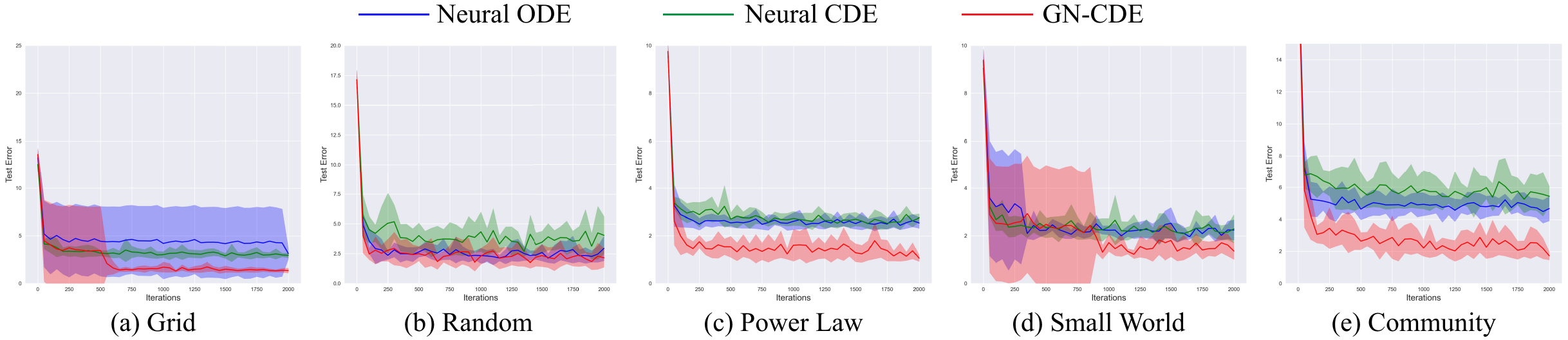}
        \vspace{-0.2cm}
		\caption{\textbf{Gene Regulation:}~The test errors of Neural ODE, Neural CDE and our GN-CDE models with respect to the optimization iteration count under five different graph structures: (a) grid, (b) Random, (c) power law, (d) small world and (e) community.}
	   \label{fig:train_curve_gene}
       % \vspace{0.4cm}
	\end{minipage}
\end{figure*}
% \vspace{-0.2cm}

% \noindent\textbf{Baselines.}~
The methods for comparison include Neural ODE (presented in Eq.~\eqref{eq:gn_cde_abb1}), Neural CDE (presented in Eq.~\eqref{eq:gn_cde_abb2}) and STG-NCDE~\cite{Choi2022AAAI}. We implement our method following Eq.~(\ref{eq:gn_cde_imp}) based on torchcde~\cite{Kidger2020CDE}, which supports differentiable CDE solvers on GPUs. A natural cubic spline is employed to interpolate the observed graph structures continuously over time. We take $\Tilde{\mathbf{A}}_s$ to be the sum of $\mathbf{A}_s$ and $\frac{\mathrm{d}\mathbf{A}_s}{\mathrm{d}s}$ for a tiny implementation. The vector field is parameterized using a single GCN layer followed by a ReLU activation function with the dimension of output node embeddings set as $d=20$. To enable node property prediction, we further attach one linear layer to acquire the final node states. Model parameters are optimized by minimizing the F1 loss between the predicted and ground truth dynamics using the Adam optimizer for $2{,}000$ iterations with an initial learning rate of $0.01$. For a fair comparison, we keep the neural network architecture of the vector field the same for all baselines. We report the mean absolute error averaged over 10 runs with the standard deviation shown aside.

\noindent\textbf{Quantitative Results.}~The results of our proposed GN-CDE model and other baseline methods in node attribute prediction tasks are summarized in Table~\ref{tab:exp_attribute}. We can observe that, our GN-CDE model consistently achieves the lowest prediction error in both heat diffusion prediction tasks and gene regulation prediction tasks across diverse dynamic networks. The improvements over the Neural ODE and Neural CDE methods are marginal but consistent. Additionally, in terms of gene regulation prediction, all approaches exhibit higher prediction errors compared to heat diffusion prediction. This discrepancy can be attributed to the inherent complexity of gene regulation dynamics, as discussed in detail in Appendix~E. Despite this, our proposed approach still outperforms the baseline methods by a significant margin. This demonstrates the effectiveness of our model in accurately modeling the whole dynamics under evolving graph structures, no matter for interpolation or for extrapolation.

\noindent\textbf{Convergence Analysis.}~We then investigate the convergence of our model by reporting the prediction errors with respect to the number of iteration steps. The results for heat diffusion and gene regulation are depicted in Fig.~\ref{fig:train_curve_heat} and Fig.~\ref{fig:train_curve_gene}, respectively. It can be seen that our GN-CDE model exhibits efficient convergence, commonly reaching a stable test error in less than $200$ iterations. The convergence speed is comparable to that of the Neural ODE and Neural CDE methods, while demonstrating significantly lower prediction errors. This confirms the strong convergence capability of our proposed method. However, we also observe a noticeable oscillation for GN-CDE in gene regulation involving a small world graph structure. This arises from a single random seed, for which almost $50\%$ of the structural changes happen during the final $20\%$ of the time span. We conjecture that the high frequency of structural changes within a short duration renders the synthesized data difficult to fit.

\begin{table}[!tbp]
\caption{Quantitative evaluation of different interpolation schemes on Heat Diffusion dynamics over the grid and random networks.}
\vspace{-0.2cm}
\label{tab:exp_scheme}
\begin{center}
\begin{normalsize}
    \begin{tabular}{p{70pt}<{\raggedright}p{65pt}<{\centering}p{65pt}<{\centering}}
    \toprule[1pt]
    Control & Grid & Random \\
    \midrule[0.6pt]
    Linear        & 0.453 $\pm$ 0.152 & 0.554 $\pm$ 0.119  \\
    Rectilinear   & 0.873 $\pm$ 0.272 & 0.777 $\pm$ 0.148   \\
    Cubic Hermite & 0.491 $\pm$ 0.075 & 0.523 $\pm$ 0.072 \\
    Natural Cubic & \textbf{0.369} $\bm\pm$ \textbf{0.134} & \textbf{0.521} $\bm\pm$ \textbf{0.202} \\
    \bottomrule[1pt]
    \end{tabular}
\end{normalsize}
\end{center}
\end{table}

\noindent\textbf{Parameter Sensitivity Analysis.}~Different interpolation schemes directly affect the smoothness of the interpolated graph path for message passing in GN-CDE. To empirically evaluate their impact on model performance, we compare GN-CDE with linear, rectilinear, natural cubic, and cubic hermite interpolation strategies. The experimental results on heat diffusion  dynamics over both grid and random networks are presented in Table~\ref{tab:exp_scheme}. As seen, natural cubic interpolation consistently achieves the best performance, yielding the lowest error on both graph types (0.369 on grid and 0.521 on random networks). This suggests that smoother and higher-order continuous interpolation can better capture the underlying temporal evolution of node embeddings in dynamic graphs. In contrast, less smooth schemes such as rectilinear interpolation result in higher errors, indicating poor capability in modeling the latent dynamics.

\begin{figure}[!tbp]
    \centering
    \includegraphics[width=0.97\linewidth]{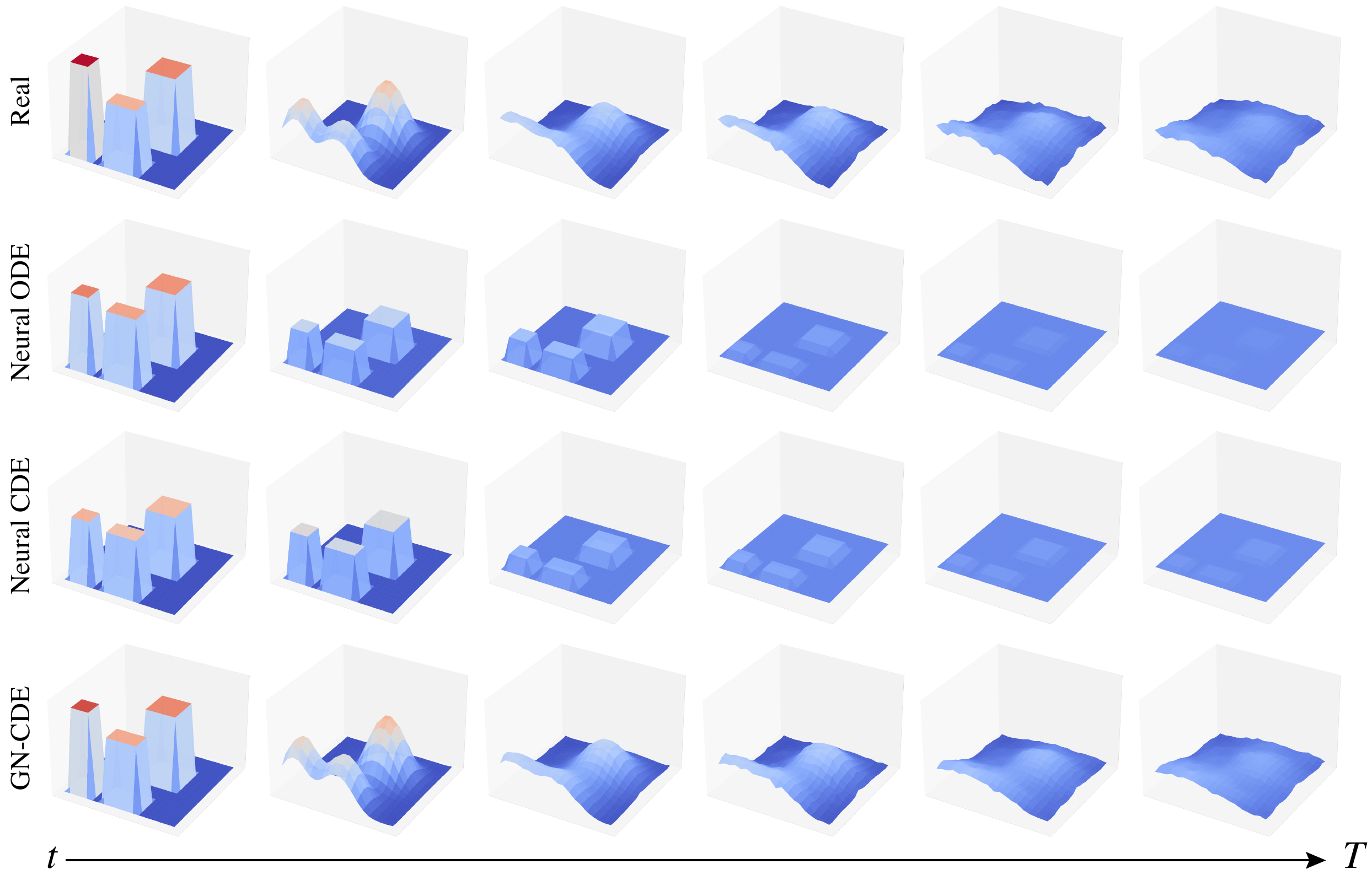}
    \vspace{-0.1cm}
    \caption{Visualization of learned dynamics for heat diffusion over dynamic graphs. Our GN-CDE model fits the dynamics for the whole progress accurately. All experiments have been averaged over 10 random seeds with the stand deviation reported. }
    \label{fig:node_attribute_viz}
\end{figure}

\noindent\textbf{Analysis on Instantaneous Structural Changes.}~
To assess the effectiveness of our approach in dealing with instantaneous structural changes, we compare the ground truth dynamics with the learned dynamics from our method on the heat diffusion task over a time-varying grid network. The results are illustrated in Fig.~\ref{fig:node_attribute_viz}. Due to the instantaneous structural changes during the process of energy diffusion towards neighboring nodes, the node attribute surface of the ground truth displays several irregularities (as shown in the first panel of Fig.~\ref{fig:node_attribute_viz}). Both the Neural ODE and Neural CDE methods fail to precisely fit the surface. This is due to their poor capability in handling structural changes. On the contrary, our GN-CDE model demonstrates favorable performance, accurately fitting the dynamics throughout the entire time span. This confirms the superiority of incorporating structural dynamics into both the vector field and differential term, particularly in handling instantaneous structural changes.

\begin{table}[!tbp]
\caption{Experimental results for dynamic node affinity prediction task on the tgbn-trade and tgbn-genre datasets. Best results are in bold, and second-best are underlined.}
\label{tab:real}
\vspace{-0.25cm}
\begin{center}
\begin{normalsize}
\begin{tabular}{p{70pt}<{\raggedright}p{65pt}<{\centering}p{65pt}<{\centering}} % p{52pt}<{\centering}p{52pt}<{\centering}p{52pt}<{\centering}
\toprule[1pt]
Algorithms & tgbn-trade & tgbn-genre \\ %& tgbn-reddit \\
\midrule[0.6pt]
JODIE~\cite{Kumar2019KDD}      & 0.374 $\pm$ 0.09  & 0.350 $\pm$ 0.04 \\ %& 0.314 $\pm$ 0.01 \\
TGAT~\cite{Xu2020ICLR}         & 0.375 $\pm$ 0.07  & 0.352 $\pm$ 0.03 \\ %& 0.314 $\pm$ 0.01 \\ 
CAWN~\cite{Wang2021ICLR}       & 0.374 $\pm$ 0.09  & - \\ %&- \\
TCL~\cite{Wang2021arXiv}       & 0.375 $\pm$ 0.09  & 0.354 $\pm$ 0.02 \\ %& 0.314 $\pm$ 0.01 \\
GraphMixer~\cite{Cong2023ICLR} & 0.375 $\pm$ 0.11  & 0.352 $\pm$ 0.03 \\ %& 0.314 $\pm$ 0.01 \\
DyGFormer~\cite{Yu2023NIPS}    & 0.388 $\pm$ 0.64  & 0.365 $\pm$ 0.20 \\ %& 0.316 $\pm$ 0.01 \\
DyRep~\cite{Trivedi2019ICLR}   & 0.374 $\pm$ 0.001 & 0.351 $\pm$ 0.001 \\ %& 0.312 $\pm$ 0.001 \\
TGN~\cite{Rossi2020ICML}       & 0.374 $\pm$ 0.001 & 0.367 $\pm$ 0.058 \\ %& 0.315 $\pm$ 0.020 \\
TGNv2~\cite{Tjandra2024TGN2}   & \textbf{0.736} $\pm$ \textbf{0.006} & 0.469 $\pm$ 0.002 \\ %& 0.507 $\pm$ 0.002 \\
\hline
Neural ODE & 0.725 $\pm$ 0.022 & 0.465 $\pm$ 0.008 \\
Neural CDE & 0.721 $\pm$ 0.018 & \underline{0.478 $\pm$ 0.022} \\
STG-NCDE~\cite{Choi2022AAAI}   & 0.625 $\pm$ 0.021 & 0.438 $\pm$ 0.038 \\
% \cdashline{1-4}
\rowcolor{lightgray!20} GN-CDE     & \underline{0.731 $\pm$ 0.017} & \textbf{0.529} $\pm$ \textbf{0.015} \\
\bottomrule[1pt]
\end{tabular}
\end{normalsize}
\end{center}
% \vskip -0.1in
\vspace{-0.1cm}
\end{table}

% ------------------------------------------------------------------
\subsection{Dynamic Link Affinity Prediction}
\noindent\textbf{Experimental Setup.}~We evaluate our framework's performance on two real-world datasets, tgbn-trade and tgbn-genre from the Temporal Graph Benchmark (TGB)~\cite{Huang2023NIPS}. The task is to predict a node’s future affinity towards other nodes based on the observed historical evolution of a temporal graph, which is particularly relevant in applications such as recommendation systems. Each dataset is split chronologically with the ratio of 70\%, 15\%, and 15\% for training, validation, and testing, respectively.

In addition to the differential equation-based Neural ODE/CDE and STG-NCDE baselines, we further incorporate state-of-the-art temporal graph neural network-based dynamic graph models for comparison. These approaches include (1) JODIE~\cite{Kumar2019KDD}; (2) TGAT~\cite{Xu2020ICLR}; (3) CAWN~\cite{Wang2021ICLR}; (4) TCL~\cite{Wang2021arXiv}; (5) GraphMixer~\cite{Cong2023ICLR}; (6) DyGFormer~\cite{Yu2023NIPS}; (7) DyRep~\cite{Trivedi2019ICLR}; (8) TGN~\cite{Rossi2020ICML}; (9) TGNv2~\cite{Tjandra2024TGN2}. We report the averaged NDCE$\mathbf{@}$10 metric over 10 random seeds for all experiments.

\noindent\textbf{Results.}~The results are summarized in Table~\ref{tab:real}. In general, neural differential equation-based methods outperform temporal graph neural network-based approaches by a large margin with one exception of TGN-V2, which is specifically tailored for this task via explicitly incorporating source-target identification into its message-passing process. Nevertheless, our proposed GN-CDE achieves a score of 0.731 on the tgbn-trade dataset, second only to TGN-V2, and the best performance of 0.529 on the tgbn-genre dataset, demonstrating its great effectiveness in capturing complex temporal dependencies and dynamic patterns in real-world scenarios.

% ==================================================================
\section{Conclusion}
\label{sec:conclusion}
In this paper, we propose a novel differential equations-based framework called GN-CDE for representation learning on continuous-time dynamic graphs. Specifically, GN-CDE creates graph paths within the controlled differential equation such that the graph structural dynamics can be naturally incorporated when conducting integration. This is a generic framework that we can apply it to solve various dynamic graph tasks and extend it to different graph types with minor modifications. Experimental results on node attribute prediction tasks across different underlying graph structures demonstrate the superiority of our proposed method compared to other baselines. In our future work, we will explore more challenging applications of dynamic graphs in real-word scenarios. 

\ifCLASSOPTIONcompsoc
  \section*{Acknowledgments}
This work is supported by the Hong Kong Innovation and Technology Commission (InnoHK Project CIMDA), the Hong Kong Research Grants Council (Project 11201825), and the Institute of Digital Medicine, City University of Hong Kong (Projects 9229503 and 9610460), in part by the Hong Kong Research Grants Council under Projects 21200522, 11200323 and 11203220, in part by the Hong Kong Innovation and Technology Commission (Project GHP/044/21SZ), and in part by City  University of Hong Kong 11207523, Sichuan Science and Technology Fund and 2025ZNSFSC0511.
\else
  \section*{Acknowledgment}
\fi

% \clearpage
% \newpage
\bibliographystyle{IEEEtranN}
{\footnotesize
\bibliography{egbib}}

\clearpage
\newpage
\appendices

\vspace{2cm}
\noindent\textbf{\huge Supplementary Materials}
\vspace{0.6cm}

\noindent\textbf{{\large Table of Contents}}
\vspace{0.1cm}

\begin{itemize}
    \item[] \hyperref[sec:app_com]{\textbf{Appendix A:}} Comparison to alternative CDE models
    \vspace{0.1cm}
    
    \item[] \hyperref[sec:app_approx]{\textbf{Appendix B:}} Feasibility of Our Proposed GN-CDE Approximation
    \vspace{0.1cm}
    
    \item[] \hyperref[sec:app_inter]{\textbf{Appendix C:}} Different Interpolation Schemes
    \vspace{0.1cm}
    
    \item[] \hyperref[sec:app_graph_types]{\textbf{Appendix D:}} Other Types of Graph
    \vspace{0.1cm}
    
    \item[] \hyperref[sec:app_additional_res]{\textbf{Appendix E:}} Details of Experimental Setup
\end
{itemize}

% ==================================================================
\section{Comparison to alternative CDE models}
\label{sec:app_com}

% ------------------------------------------------------------------
\subsection{Proof of Theorem~\ref{theo:gn_cde_linear}}
\label{sec:app_proof_gn_cde_linear}
Provide a dynamic graph comprising of a sequence of graph snapshots $\mathcal{G}=\{(t_0, G_{t_0}),...,(t_N, G_{t_N})\}$, with each $t_k \in \mathbb{R}$ the time stamp of the observed graph $G_{t_k}$ and $t_0<\cdot\cdot\cdot <t_N$. Besides, each graph $G_{t_k}$ is endowed with an adjacency matrix $\mathbf{A}_{t_k}$ representing the topological information for $G_{t_k}$. Further, let $\hat{\mathbf{A}}_s$ be some continuous interpolation of $\mathbf{A}_{t_k}$ such that 
$\hat{\mathbf{A}}_{t_k}=(t_k, \mathbf{A}_{t_k})$, then our GN-CDE model can be defined by

\begin{equation}
\label{eq:app_gn_cde}
    \mathbf{Z}_{t_0}=\zeta_\theta(t_0, \mathbf{A}_{t_0}), \quad \mathbf{Z}_t = \mathbf{Z}_{t_0} + \int^t_{t_0} f_{\theta}(\mathbf{Z}_s, \mathbf{A}_s)~\mathrm{d} \hat{\mathbf{A}}_s
\end{equation}
for $t\in(t_0,t_N]$. Let $\beta_s = \begin{bmatrix} \mathbf{Z}_s \\ \mathbf{A}_s\end{bmatrix}$, according to Eq.~\eqref{eq:app_gn_cde} we have 
\begin{equation}
\label{eq:ncde}
    \beta_t = \begin{bmatrix} \mathbf{Z}_t \\ \mathbf{A}_t \end{bmatrix} = \begin{bmatrix} \mathbf{Z}_{t_0} \\ \mathbf{A}_{t_0} \end{bmatrix} + \int^t_{t_0} \begin{bmatrix} f_{\theta}(\mathbf{Z}_s, \mathbf{A}_s) \\ 1 \end{bmatrix} ~\mathrm{d} \hat{\mathbf{A}}_s
\end{equation}
for $t\in(t_0,t_N]$. Then we can let $\tilde{f}_{\theta}(\beta_s)=\begin{bmatrix} f_{\theta}(\mathbf{Z}_s, \mathbf{A}_s) \\ 1 \end{bmatrix}$, the above equation can be rewritten as
\begin{equation}
    \beta_t = \beta_{t_0} + \int^t_{t_0} \tilde{f}_{\theta}(\beta_s) ~\mathrm{d} \hat{\mathbf{A}}_s
\end{equation}
for $t\in(t_0,t_N]$. This formulation is equivalent to the original Neural CDE formulation presented in~\citet{Kidger2020CDE}, then we accomplish the proof.

\begin{table*}[htbp]
\caption{The comparison of different interpolation schemes.}
\vspace{-0.2cm}
\label{tab:com_interpolation}
% \vskip 0.1in
\begin{center}
\begin{normalsize}
\begin{tabular}{p{70pt}<{\centering}p{60pt}<{\centering}ccc}
    \toprule[1pt]
    \multirow{2}*{\shortstack[l]{\textbf{Interpolation} \\ \quad\textbf{Schemes}}} &  \multicolumn{4}{c}{\textbf{Properties}}  \\
    \cline{2-5}
    ~ &  Smoothness  & Dependency on Future & Interpolation Complexity  & Integral Difficulty  \\
    \midrule
    Linear        & (piecewise) & One & Low     & High \\
    Rectilinear   & (piecewise) & No  & Lowest  & High \\
    Natural Cubic & \checkmark  & All & Highest & Lowest \\
    Cubic Hermite & \checkmark  & One & High    & Low \\
    \bottomrule[1pt]
\end{tabular}
\end{normalsize}
\end{center}
\end{table*}

% ==================================================================
\section{Feasibility of Our Proposed GN-CDE Approximation}
\label{sec:app_approx}

\subsection{Maximal Expressivity and the Signature}

\begin{definition}[Maximal expressivity \cite{walker2025structured}]
\label{def:universal_approximation_general}
Let $\mathcal{X}$ be a topological space, and let 
$\mathcal{F} = \{ f_\theta : \mathcal{X} \to \mathbb{R} \mid \theta \in \Theta \}$
be a class of real-valued functions on $\mathcal{X}$, parametrised by some set $\Theta$. 
We say that $\mathcal{F}$ is maximally expressive (or universal) if, for every compact set $\mathcal{K} \subset \mathcal{X}$ and every real-valued continuous function $f : \mathcal{K} \to \mathbb{R}$, the following property holds:
\begin{equation}
\forall \epsilon > 0, \; \exists \theta \in \Theta \quad \text{s.t.} \quad 
\sup_{x \in \mathcal{K}} \big| f(x) - f_\theta(x) \big| < \epsilon.
\end{equation}
\end{definition}

Due to the universality of the signature, a continuous time series model is maximally expressive if, for all $N$ and $t$, there exists a model configuration such that the output path can be arbitrarily close to $S^N_t$ \cite{Lyons2014}. Given a path $X:[0,T]\rightarrow\mathbb{R}^e$, the truncated signature $S^N \in T^N(\mathbb{R}^e)$ of a path $X:[0,T]\rightarrow\mathbb{R}^e$ solves
\begin{equation}
    \label{eq:sig_tensor_cde}
    \mathrm{d}S^N_s = (S^N_s \otimes \mathrm{d}X_s)^{\leq N},
\end{equation}
where $S^N_0 = [1,0,\ldots,0]$. For a word $w=i_1\cdots i_k$, let 
\begin{equation}
    \text{index}(w) = 1 + \sum_{j=1}^ki_je^{k-j}.
\end{equation}
Then, letting $a=\text{index}(w)$, an equivalent formulation of \eqref{eq:sig_tensor_cde} is
\begin{equation}
    \mathrm{d}S^N_{a,s} = A_{abc} S^N_{c,s} \mathrm{d}X_{b,s},
\end{equation}
where $A_{abc}$ is defined by
\begin{equation}
    A_{abc} = \begin{cases}
        1, \quad &a=1+b+e(c-1), \\
        0, &\text{otherwise},
    \end{cases}
\end{equation}
for $b\in\{1,\ldots,e\}$ and $a,c\in \{1,\ldots,\kappa(e,N)\}$, where $\kappa(e,N)$ is the dimension of a depth$-N$ truncated signature of an $e-$dimensional path. 

\subsection{Proof of Theorem \ref{theo:gn_cde_approx}}

As shown in \cite{Kidger2020CDE}, NCDEs can approximate the truncated signature of their input path arbitrarily closely. Therefore, the GN-CDE,
\begin{equation}
\label{eq:app_gn_cde_deri}
    \mathbf{Z}_t = \mathbf{Z}_{t_0} + \int^t_{t_0} f_{\theta}(\mathbf{Z}_s, \mathbf{A}_s)\mathrm{d}\hat{\mathbf{A}}_s,
\end{equation}
is maximally expressive. We now show the same is true given a set of independent copies of the model
\begin{equation}
\label{equ:gcn_vf}
    \mathbf{Z}_t = \mathbf{Z}_{t_0} + \int^t_{t_0} \mathbf{W}^{(1)}\mathrm{d}\hat{\mathbf{A}}_s \mathbf{W}^{(2)} \mathbf{Z}_s\mathbf{W}^{(L)},
\end{equation}
which is a specific version of our approximation \eqref{eq:gn_cde_imp}, where $\sigma$ is the identity, $\mathbf{W}^{(DR1)}=\mathbf{W}^{(1)}$ and $\mathbf{W}^{(DR2)}=[0, \mathbf{W}^{(2)}]$.

In order to handle the independent copies as one system, we introduce an index $r\in\{1,\ldots,R\}$ and expand the driving path to $\bar{\mathbf{A}}_s=I_R\otimes \hat{\mathbf{A}}_s\in\mathbb{R}^{R\times R \times |\mathcal{V}| \times |\mathcal{V}|}$ and the hidden state to $\tilde{\mathbf{Z}}_s \in\mathbb{R}^{R\times R \times |\mathcal{V}|\times d}$. The sum of the independent GN-CDE copies is given by
\begin{equation}\label{eq:lift}
  \mathrm{d}Z_{ij,t}=\tilde{\mathbf{W}}_{rrik}^{(1)}\mathrm{d}\tilde{X}_{rrkl,t}\tilde{\mathbf{W}}^{(2)}_{rrlm}\,\tilde{Z}_{rrmo,t}\tilde{\mathbf{W}}_{rroj}^{(L)}\mathbf{I}_{rr},
\end{equation} 
where $\mathbf{I}\in\mathbb{R}^{R\times R}$ is the identity matrix,
\begin{equation}
  \tilde{\mathbf{W}}^{(1)}
    =\sum_{r=1}^R E_{rr} \otimes \mathbf{W}^{(1),r}\in\mathbb{R}^{R\times R \times |\mathcal{V}| \times |\mathcal{V}|}
\end{equation}
\begin{equation}
  \tilde{\mathbf{W}}^{(2)}
    =\sum_{r=1}^R E_{rr} \otimes \mathbf{W}^{(2),r}\in\mathbb{R}^{R\times R \times |\mathcal{V}| \times |\mathcal{V}|},
\end{equation}
\begin{equation}
  \tilde{\mathbf{W}}^{(L)}
    =\sum_{r=1}^R E_{rr} \otimes \mathbf{W}^{(L),r}\in\mathbb{R}^{R\times R \times d\times d},
\end{equation}
and $E_{rr}$ is the $R\times R$ matrix with a one at $(r,r)$.
Define the row-major flattening maps
\begin{equation*}
  \operatorname{vec}_{|\mathcal{V}|}\colon\mathbb{R}^{|\mathcal{V}|\times d}\to\mathbb{R}^{|\mathcal{V}|d},
  \qquad
  Z_{ij}\longmapsto z_{a},\;a=i+|\mathcal{V}|(j-1),
\end{equation*}
\begin{equation*}
  \operatorname{vec}_{|\mathcal{V}|^{2}}\colon\mathbb{R}^{|\mathcal{V}|\times |\mathcal{V}|}\to\mathbb{R}^{|\mathcal{V}|^{2}},
  \qquad
  X_{kl}\longmapsto x_{b},\;b=k+|\mathcal{V}|(l-1).
\end{equation*}
Setting $z_t=\operatorname{vec}_{|\mathcal{V}|}(Z_t)$ and $x_t=\operatorname{vec}_{|\mathcal{V}|^{2}}(X_t)$,
\begin{equation}
\label{eq:exp_cde}
    \mathrm{d}z_{a,t} = \tilde{B}_{abc} z_{c,t}\mathrm{d}x_{b,t}, \quad\quad \tilde{B}_{abc} = \sum_{r=1}^RW^{3,r}_{ik}W^{1,r}_{lm}W^{2,r}_{oj}.
\end{equation}
Take $R$ to be the number of non-zero entries in $A_{abc}$, i.e. number of triples $(a,b,c)$ which satisfy $a=1+b+n^2(c-1)$, and for each triple $(a_r,b_r,c_r)$ let $(i_r,j_r,k_r,l_r,m_r,o_r)$ be the corresponding indices. Taking $d$ such that $nd>\kappa(n^2,N)$ and
\begin{equation}
    W^{1,r} = E_{l_rm_r}, \quad W^{2,r} = E_{o_rj_r} \quad W^{3,r} = E_{i_rk_r},
\end{equation}
then,
\begin{equation}
    \tilde{B}_{abc} = \begin{cases} A_{abc}, \quad 1\leq a,c \leq \kappa(n^2,N), 1\leq b \leq n^2, \\
    0 \quad \text{ otherwise.}
    \end{cases}
\end{equation}
Finally, taking $z_0=(1,0,\ldots,0)$, the first $\kappa(n^2,N)$ entries of $z_t$ will be the depth $N$ truncated signature and all other elements will be zero.

% ==================================================================
\section{Different Interpolation Schemes}
\label{sec:app_inter}
We consider four different interpolation schemes for our GN-CDE model: 1) Linear control; 2) Rectilinear control; 3) Natural cubic splines; 4) Cubic Hermite splines with backward differences. As we have analyzed the continuity of vector field $f_{\theta}(\mathbf{Z}_s, \mathbf{A}_s)$ and the existence and uniqueness of solutions for our framework in Section.~\ref{sec:gn_cde}, in this part, we analyze the practical performance of them in the smoothness property, interpolation complexity and optimization difficulty when utilizing some ODE solvers (~\eg~Euler method, Dormand-Prince (DOPRI) method~\citep{Dormand1980DOPRI}) for the integral. We summarize the results in Table~\ref{tab:com_interpolation}.

Commonly, the ODE solvers calculate Eq.~\eqref{eq:gn_cde_deri} use another form as
\begin{equation}
\label{eq:gn_cde_diff}
    \frac{\mathrm{d} \mathbf{Z}_t}{\mathrm{d} t} = f_{\theta}(\mathbf{Z}_t, \mathbf{A}_t) \frac{\mathrm{d}\hat{\mathbf{A}}t}{\mathrm{d}t}.
\end{equation}
We also utilize this formula to show different effects caused by the interpolation schemes for our GN-CDE model.

\noindent\textbf{Linear control.}~If we have two observations $(t_0, \mathbf{A}_{t_0})$ and $(t_2, \mathbf{A}_{t_2})$ collected at time stamps $t_0$ and $t_2$, respectively, and we want the get the value for time $t_1$. Linear control is the interpolating along the straight line between these two observations. Formally, $\mathbf{A}_{t_1}$ can be evaluated by solving the equation
\begin{equation*}
    \frac{\mathbf{A}_{t_1} - \mathbf{A}_{t_0}}{t_1 - t_0} = \frac{\mathbf{A}_{t_2} - \mathbf{A}_{t_1}}{t_2 - t_1}.
\end{equation*}
For our GN-ODE model, the vector field $f_{\theta}(\mathbf{Z}_t, \mathbf{A}_t)$ in Eq.~\eqref{eq:gn_cde_diff} is implemented as GCN layers with the formula $f_{\theta}(\mathbf{Z}_s, \mathbf{A}_s)=\sigma(\mathbf{A}_s \mathbf{Z}_s^{(l)} W^{(l)})$, it is Lipschitz continuous. While for the derivation $\frac{\mathrm{d}\hat{\mathbf{A}}t}{\mathrm{d}t}$, its value will be a constant when there exists a graph structural change otherwise zero. The multiplication of these two terms will exhibit some jumps due to the gradient discontinuities at the structural changing moment $(t_k, \mathbf{A}_{t_k})$. Thus, this scheme owns moderate complexity for interpolation and high integral difficulty for the solvers to resolve the jumps.

\noindent\textbf{Rectilinear control.}~For the observations $\{(t_0, \mathbf{A}_{t_0}),...,(t_N, \mathbf{A}_{t_N})\}$, rectilinear control updates the time and feature channels separately in lead-lag fashion as $\hat{\mathbf{A}}_t:[0, 2n] \rightarrow \mathbb{R}^{|\mathcal{V}| \times |\mathcal{V}|\times 2}$ such that $\hat{\mathbf{A}}_{2k}=(t_k, \mathbf{A}_{t_k})$ and $\hat{\mathbf{A}}_{2k+1}=(t_{k+1}, \mathbf{A}_{t_k})$. 

\begin{figure*}[htbp]
\centering
\includegraphics[width=0.76\textwidth]{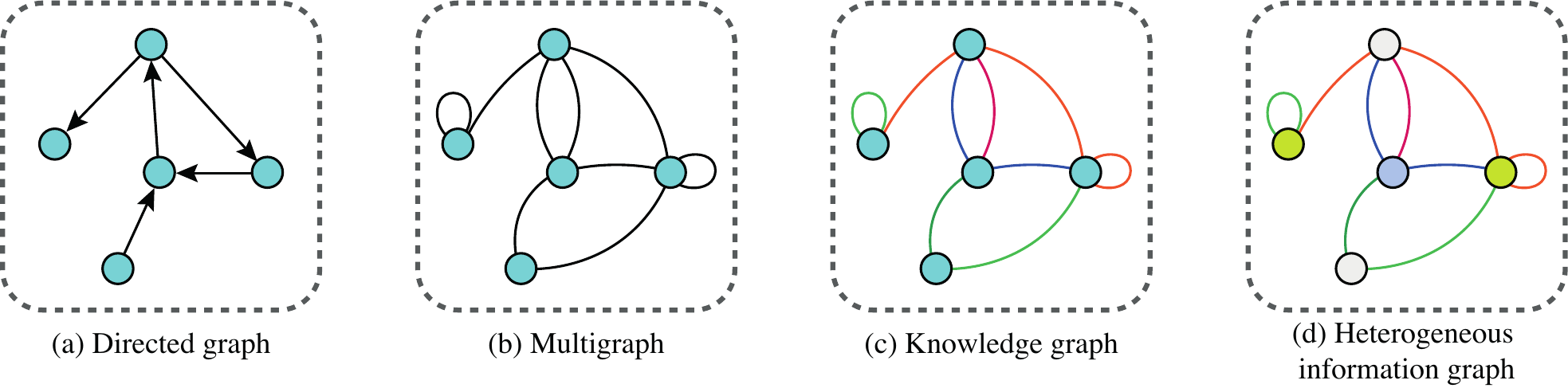}
\vspace{-0.2cm}
\caption{Illustration of other graph types.}
\label{fig:graph_types}
\end{figure*}

When using rectilinear control as the interpolation scheme to build the path, the derivation $\frac{\mathrm{d}\hat{\mathbf{A}}_t}{\mathrm{d}t}$ for each possible interaction could appear with a large value when there exists a graph structural change. This leads to the derivation in Eq.~\eqref{eq:gn_cde_diff} much more non-smooth than linear control. Besides, the length of the created path is twice as long as other schemes, thus it takes a longer time for our model to evaluate and train. The merit of this scheme is its low interpolation complexity since it only needs to pad the value for a time with the previous one.

\noindent\textbf{Natural cubic splines.}~The natural cubic spline is a spline that uses the third-degree polynomial which satisfies the given control points. This is also the recommended interpolation scheme by \cite{Kidger2020CDE}.

With natural cubic splines, the created paths are twice differentiable. This means the ODE solvers do not need to resolve the jumps of $\frac{\mathrm{d} \mathbf{Z}_t}{\mathrm{d} t}$ since the derivation is smooth enough. Thus, the integral difficulty for this scheme is the lowest among the four discussed schemes. However, this scheme can lead to unacceptable time costs for path creation, especially when the scale of the graph is large.

\noindent\textbf{Cubic Hermite splines with backward differences.}~For each interval $[t_k, t_{k+1})$ in among the whole observations $\{(t_0, \mathbf{A}_{t_0}),...,(t_N, \mathbf{A}_{t_N})\}$, cubic Hermite spline keeps that $\hat{\mathbf{A}}_{t_k}=(t_{k}, \mathbf{A}_{t_k})$ and $\hat{\mathbf{A}}_{t_{k+1}}=(t_{k+1}, \mathbf{A}_{t_{k+1}})$, and that the gradient at each node matches the backward finite difference as
\begin{equation*}
    \frac{\mathrm{d} \hat{\mathbf{A}}_{t_k}}{\mathrm{d}t} = \hat{\mathbf{A}}_{t_k} - \hat{\mathbf{A}}_{t_{k-1}}
    \quad\text{and}\quad
    \frac{\mathrm{d} \hat{\mathbf{A}}_{t_{k+1}}}{\mathrm{d}t} = \hat{\mathbf{A}}_{t_{k+1}} - \hat{\mathbf{A}}_{t_k}.
\end{equation*}

This scheme can smooth gradient discontinuities of linear control, thus its integration difficulty is lower than linear control. Besides, since it only needs to solve a single equation on each control point, its interpolation complexity is lower than natural cubic splines. However, the spurious delays in the spline will degrade the accuracy of our model. 

See \cite{Morrill2022TMLR} for more descriptions of these schemes.

% ==================================================================
\section{Other Types of Graph}
\label{sec:app_graph_types}

Our proposed GN-CDE model is not only a generic framework for tackling different graph-related tasks, such as the prediction of node attributes, dynamic nodes classification and temporal link prediction but can also be easily extended to more complex graph structures. We discuss how our model can be applied to solve tasks on directed graphs, multigraph, knowledge graphs, and heterogeneous information networks. An illustration of these graph structures is presented in Fig.~\ref{fig:graph_types}.

\noindent\textbf{Directed graph.}~Vertices in a directed graph are connected by directed edges (See Fig.~\ref{fig:graph_types}(a)). In order to apply our GN-CDE model to tasks established upon directed dynamic graphs, it is necessary to modify our implementation of the vector field from graph convolutional layers defined for undirected graphs to directed graphs. In addition to utilizing an asymmetric adjacency matrix directly, we can also represent the graph structure with spectral-based methods that leverage edge direction proximity~\cite{Tong2020arXiv}, transforms~\cite{Sardellitti2017STSP, Zhang2021NIPS} or local graph motifs~\cite{Monti2018DSW}.

\noindent\textbf{Multigraph.}~A graph in which there are multiple edges between two nodes (See Fig.~\ref{fig:graph_types}(b)). The most common operations for performing graph convolution on a multigraph are graph fusion and separate subgraphs~\citep{Zhou2020Review}. These techniques can cooperate with our GN-CDE model for representation learning on evolving multigraph.

\noindent\textbf{Knowledge graph.}~This graph is a collection of real-world entities and the relational facts between pairs of entities. The underlining graph structure of the knowledge graph is commonly a multi-digraph with labeled edges, where the labels indicate the types of relationships (See Fig.~\ref{fig:graph_types}(c)). In order to learn the graph embeddings, we would first utilize the graph fusion or separate subgraphs technique proposed for processing multigraph to acquire a single graph, then we employ the GN-CDE model presented in Eq.~\eqref{eq:gn_cde_node} but without node attributes matrix for the continuous inference, that is
\begin{equation}
    \mathbf{Z}_t = \mathbf{Z}_{t_0} + \int^t_{t_0} f_{\theta}\Big(\mathbf{Z}_s, \mathbf{A}_s, \mathbf{E}\Big) \mathrm{d}\hat{\mathbf{A}}_s \quad \mathrm{for}~t\in(t_0,t_N],
\end{equation}
where $\mathbf{Z}_{t_0}=\zeta_\theta(t_0, \mathbf{A}_{t_0}, \mathbf{E})$.

\noindent\textbf{Heterogeneous information networks.}~This is a complex graph type that consists of multiple types of nodes or edges (See Fig.~\ref{fig:graph_types}(d)). In order to deal with the dynamic node types and edge features of heterogeneous information networks, we can directly apply the GN-CDE model presented in Eq.~\eqref{eq:gn_cde_node}.

We recommend a comprehensive overview of graph neural networks~\citep{Zhou2020Review} for readers who are interested in exploring the extension of our GN-CDE model to various graph types.

% ==================================================================
\section{Details of Experimental Setup}
\label{sec:app_additional_res}

\subsection{Heat Diffusion Dynamics}
\noindent\textbf{} The dynamics are governed by Newton's law of cooling as follows,
\begin{equation}
\label{eq:heat}
    \frac{\mathrm{d} \bm{x}_t(v_i)}{\mathrm{d} t} = - k^{(i,j)} \sum_{j=1}^n \textbf{A}^{(i,j)}\left(\bm{x}(v_i) - \bm{x}(v_j)\right),
\end{equation}
where $\mathbf{x}_t(v_i)$ represents the state of node $v_i$ at time $t$ and $\textbf{A}^{(i,j)}$ is the heat capacity matrix represents the neighbors of each node $v_i$. 

\noindent\textbf{Data generation.}~To generate the data for our experiments, we first initialize a graph network with $400$ nodes using a network generator for one of Grid network, Random network, Power-law network, Small-world network and Community network. After that, we randomly select $10$ time stamps from $(t_0, T]$ as the occurrence time of structural changes. To mimic the dynamic structural changes, we randomly add or remove some edges with a uniform probability $p=0.02$ at these time stamps. Provided the initial energy for all nodes (the value of $\mathbf{F}_{t_0}$), we let the graph structure during its persistent period as a static graph, and numerically solve the heat diffusion system of Eq.~\eqref{eq:heat} segment by segment, using the Dormand–Prince method. This allows us to simulate the continuously dynamic evolution of the node attributes over time, taking into account the structural changes in the graph.

% ------------------------------------------------------------------
\subsection{Gene regulatory dynamics}
The dynamics for gene regulatory networks are governed by Michaelis-Menten equation as follows,
\begin{equation}
\label{eq:gene}
    \frac{\mathrm{d} \bm{x}_t(v_i)}{\mathrm{d} t} = -b_i \bm{x}(v_i)^f + \sum_{j=1}^n \mathbf{A}^{(i,j)} \frac{\bm{x}^h(v_j)}{\bm{x}^h(v_j) + 1},
\end{equation}
where the first term models the degradation when $f=1$ or dimerization when $f=2$, and the second term represents genetic activation, with the Hill coefficient $h$ determining the level of cooperation in the regulation of the gene.

\noindent\textbf{Data generation.}~The data generation process for gene regulation is similar to that of heat diffusion, except that the differential equation used to model the dynamic system is replaced by Eq.~\eqref{eq:gene}. 

% ------------------------------------------------------------------
\subsection{tgbn-trade}
This dataset from TGB~\cite{Huang2023NIPS} represents the international agricultural trading network among United Nations (UN) member nations from 1986 to 2016, comprising 255 nations (nodes) and 468,245 trade relationships (edges). Each edge denotes the total trade value of all agricultural products exported from one nation to another within a given year. As the data is reported annually, the time granularity of the dataset is yearly. The primary  task associated with this dataset is to predict the proportion of agricultural trade values from one nation to other nations in the subsequent year.

% ------------------------------------------------------------------
\subsection{tgbn-genre}
This dataset from TGB~\cite{Huang2023NIPS} is a bipartite, weighted interaction network between users and music genres, comprising 1,505 nodes and 17,858,395 edges across 133,758 time steps. Each node represents either a user or a music genre, and each edge indicates that a user listened to a genre at a specific time, with the edge weight reflecting the percentage of a song’s association with that genre. The dataset integrates listening records and genre metadata, retaining only genres that account for at least 10\% genre weight per song and that appear at least 1,000 times. Genre names have been standardized to remove typos. The primary task is to predict the frequency with which each user will interact with music genres over the subsequent week, facilitating personalized music recommendations as user preferences evolve over time.

\ifCLASSOPTIONcaptionsoff
  \newpage
\fi

\end{document}